\documentclass[10pt,twocolumn]{article}
\usepackage[utf8]{inputenc}
\usepackage[letterpaper,margin=1in]{geometry}
\usepackage{microtype}
\usepackage{lmodern}
\usepackage[T1]{fontenc}

\usepackage{authblk}
\usepackage{amsmath,amssymb,amsfonts,mathrsfs, amsthm}
\usepackage{mathtools}
\usepackage{bm}
\usepackage{graphicx}
\usepackage{subcaption}
\usepackage{natbib}
\usepackage{hyperref}
\usepackage{url, bbm}
\usepackage{cleveref}
\usepackage{algorithm, algorithmic}
\crefformat{equation}{#2(#1)#3}
\Crefformat{equation}{#2(#1)#3}
\usepackage{enumitem}
\usepackage[rgb]{xcolor}
\usepackage{tikz}
\usepackage{tkz-graph}
\usepackage{authblk}
\usetikzlibrary{shapes.geometric}
\usetikzlibrary{backgrounds}
\usetikzlibrary{arrows.meta}
\usepackage[framemethod=TikZ]{mdframed}
\numberwithin{equation}{section}

\theoremstyle{plain}
\newtheorem{theorem}{Theorem}[section]
\newtheorem{example}[theorem]{Example}

\newtheorem{lemma}[theorem]{Lemma}
\newtheorem{proposition}[theorem]{Proposition}
\newtheorem{assumption}{Assumption}

\theoremstyle{definition}

\theoremstyle{remark}
\newtheorem{remark}[theorem]{Remark}

\DeclareMathOperator*{\argmin}{arg\,min}

\DeclarePairedDelimiterX{\norm}[1]{\lVert}{\rVert}{#1}
\DeclarePairedDelimiterX{\abs}[1]{\lvert}{\rvert}{#1}

\renewcommand{\epsilon}{\varepsilon}

\newcommand{\rank}[1]{\operatorname{rank}\!\left(#1\right)}

\newcommand{\independent}{\perp\!\!\!\perp}

\newcommand{\var}{\operatorname{Var}}
\newcommand{\cov}{\operatorname{Cov}}
\newcommand{\tr}{\operatorname{tr}}

\newcommand\ti[1]{{\tilde{#1}}}

\newcommand\ivs{{\hat{\beta}^{\mathrm{UP}}_{\mathrm{GMM}, \ell_1}}}
\newcommand\ivsf{{\tilde\beta}}
\newcommand\gmm{{\hat{\beta}^{\mathrm{UP}}_{\mathrm{GMM}}}}

\let\leq\leqslant
\let\geq\geqslant
\let\le\leqslant
\let\ge\geqslant

\newcommand\estn{{\textsc{SplitUP}}}
\newcommand\estnfd{{\textsc{UP-GMM}}}

\definecolor{col1}{RGB}{88,140,126}
\definecolor{col2}{RGB}{242,227,148}
\definecolor{col3}{RGB}{242,174,114}
\definecolor{col4}{RGB}{217,100,89}
\definecolor{col5}{RGB}{140,70,70}
\colorlet{lightgray}{black!15}

\title{
Many Experiments, Few Repetitions, Unpaired Data, 
  and Sparse Effects: Is Causal Inference Possible? }

\author[1]{Felix Schur\thanks{Corresponding author: felix.schur@stat.math.ethz.ch}}
\author[2]{Niklas Pfister}
\author[3]{Peng Ding}
\author[4,5]{Sach Mukherjee}
\author[1]{Jonas Peters}

\affil[1]{Department of Mathematics, ETH Zurich}
\affil[2]{Lakera AI}
\affil[3]{Department of Statistics, UC Berkeley}
\affil[4]{German Center for Neurodegenerative Diseases (DZNE) \& University of Bonn}
\affil[5]{MRC Biostatistics Unit, University of Cambridge}

\date{}

\begin{document}

\maketitle

\begin{abstract}
    We study the problem of estimating causal effects under hidden confounding in the following unpaired data setting: we observe some covariates $X$ and an outcome $Y$ under different experimental conditions (environments) but do not observe them jointly -- we either observe $X$ or $Y$. Under appropriate regularity conditions, the problem can be cast as an instrumental variable (IV) regression with the environment acting as a (possibly high-dimensional) instrument. When there are many environments but only a few observations per environment, standard two-sample IV estimators fail to be consistent. We propose a GMM-type estimator (\estn{}) based on cross-fold sample splitting of the instrument–covariate sample and prove that it is consistent as the number of environments grows but the sample size per environment remains constant. We further extend the method to sparse causal effects via $\ell_1$-regularized estimation and post-selection refitting. 
\end{abstract}

\section{Introduction} \label{sec:intro}
\subsection{A Motivating Example} \label{sec:motivex}
\paragraph{An experiment with unpaired data.} 
Suppose that we have some response $Y$ (e.g., a phenotype), 
some covariates $X$ (e.g., a genotype)
and that we observe the system under $m$ different experimental conditions (e.g., using gene modification technology).
Suppose we are interested in estimating the causal effect from $X$ to $Y$ but that we are facing the following two challenges: 
In each experiment, we may observe multiple i.i.d.\ repetitions but for each repetition, we can either measure $X$ or $Y$ but not both (see \Cref{tab:intro}) 
\begin{table}[!htbp]
    \centering
\small
\begin{tabular}{c||c|c|c|c|c|c|c|c}
Exp. & $1$ & $1$ & $\cdots$ & $2$ & $2$ & $2$ & $\cdots$ & $m$\\\hline \hline
$X_1$  & $1.8$ & $\times$ & $\cdots$ & $3.2$ & $\times$ & $\times$ & $\cdots$ & $0.6$\\\hline
$\vdots$ & $\vdots$ &$\vdots$ &$\vdots$ &$\vdots$ &$\vdots$ &$\vdots$ & $\vdots$ & $\vdots$ \\\hline
$X_d$  & $2.0$ & $\times$ & $\cdots$ & $0.2$ & $\times$ & $\times$ & $\cdots$ & $1.0$\\\hline
$Y$  & $\times$ & $2.5$ & $\cdots$ & $\times$ & $-3.7$ & $1.1$ & $\cdots$ & $\times$
\end{tabular}
    \caption{
    We consider the problem of estimating the causal effect between $X$ and $Y$ in the presence of hidden confounding when data are unpaired. The table shows an example dataset, where we observe the system under $m$ different experimental conditions.}
    \label{tab:intro}
\end{table}
and there may be unobserved confounding between $X$ and $Y$.
It may come as a surprise that in such a setting, consistent estimation of causal effects is 
possible under weak assumptions.

\paragraph{IV and regression of the means.}
These assumptions are satisfied, 
e.g., 
if the experimental conditions influence $X$ but do not influence $Y$ directly or via another causal path (we will make this precise in \Cref{sec:identifsparse}). 
This condition is well-studied in the literature on instrumental variables (IV) \citep{AngristKrueger1992, angrist1995identification, angrist1996identification, StaigerStock1997}.
Let us assume for a moment that for each experimental condition $j \in \{1, \ldots, m\}$, we observe all $X$ and all $Y$, so the data are paired. 
We can then apply IV 
estimators to the example above when using a one-hot encoding for the different experimental conditions. 
In this case (see \Cref{sec:categor} for a derivation), the classical two-stage least squares (TSLS) estimator corresponds to the following simple procedure:
within each experimental condition, we
separately average the $X$'s and $Y$'s,
resulting in a single (average) value for $X$ and a single value for $Y$ per experiment;
we then regress the $m$ values containing  $Y$ averages on the $m$ values containing $X$ averages.
Clearly, this procedure does not rely on the $X$ and $Y$ values to be paired, so it is sill applicable in the unpaired setting.
This then is not an  instantiation of IV anymore but of two-sample IV (\textsc{TS-IV}) \citep{AngristKrueger1992, inoue2010two, burgess2013summ},
a modification of IV, which has been developed 
for unpaired data (without a special focus on categorical instruments).
We regard these two algebraically trivial insights as important as the simple estimator of 
taking the average per experiment and then performing regression on the average is
often considered in practice -- both in paired and unpaired settings \citep{deaton1985panel, greenland1992methods, king2004ecological, peres2017linking}. The reasoning above provides a precise interpretation
of the target of inference: 
if the experimental condition
does not modify the causal effect, the averaging allows us to remove hidden confounding.

\paragraph{More repetitions versus more observations.}
Applying standard theory from instrumental variables now provides us with guarantees
such as consistent estimation if the number of observations per experiment converges to infinity. 
In practice, however, we may instead see many experiments and  
few repetitions per experiment, a setting that is not well-covered by these asymptotics. 
In this work, we propose a different estimator (\estn{})
that targets a setting, where
the number 
$m$ of experimental conditions (i.e., the number of instruments) 
is growing but the number of repetitions per experiment remains constant. 
We prove identifiability in the setting with $m \rightarrow \infty$ and prove consistency of the estimator. 
\Cref{fig:intro} 
shows that 
\estn{}
outperforms 
classical estimators such as \textsc{TS-IV}
on finite data. 
\begin{figure}[!htbp]
    \centering
        \includegraphics[width=0.99\linewidth]{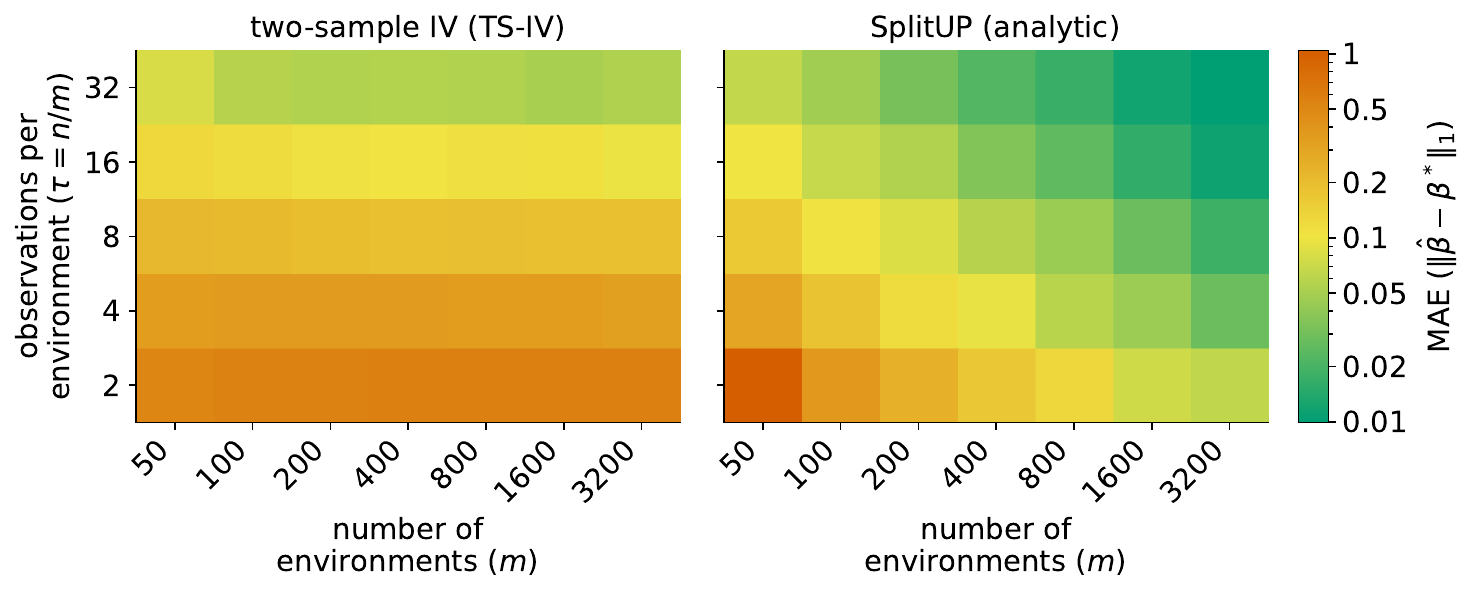}
    \caption{Estimating causal effects in a data set of the form of \Cref{tab:intro}. Our proposed estimator \estn{} (right) outperforms existing methods (left) if the number $m$ of experimental conditions increases. Indeed, we prove identifiability in the setting $m \to \infty$ in Section~\ref{sec:identifsparse}. 
    In Section~\ref{sec:estimation}, we show that unlike existing estimators (such as TS-IV on the left), \estn{} 
    is consistent as $n,m \to \infty$ and $n/m \to r \in (0,\infty)$.}
    \label{fig:intro}
\end{figure}

\paragraph{Motivation from biology.}
Applications in which data are obtained in an unpaired fashion under different causal environments abound, particularly in the biomedical domain. In contemporary biology, it is possible to measure molecular variables (such as gene or protein expression levels) at large-scale and under different causal regimes (such as defined interventions at the gene or protein level or via chemical compounds) 
and such data are the focus of much ongoing work in biotechnology and machine learning \citep[see, e.g.,][]{replogle2022,lopez2022large,lagemann2023}. Furthermore, 
it is increasingly common to consider many experimental conditions (e.g. intervention on many different genes).
Applied interest often focuses on the effects of such interventions on concrete phenotypes (such as growth rate or cellular behaviour) and it is therefore common to study phenotypes under intervention (`screening'). However, often cells have to be disturbed or destroyed in order to obtain gene or protein measurements, hence in studies in which both gene/protein levels $X$ and phenotypes $Y$ are of interest, the actual cells involved may be different and the data are therefore unpaired in the sense of this paper. %
These issues are also important in molecular medicine, where interventions on laboratory cells are used to shed light on patient responses to the same interventions (such as studying the effect of multiple cancer drugs on cells in the lab and linking these data to treatment outcomes \citep{yang2012GDSC,kirkham2025}). In such settings, the data are typically unpaired since the actual cells involved in the various steps are distinct.

\paragraph{Continuous instruments and many covariates.}
The estimators and theoretical results that we develop in this paper
apply to continuous instruments, too. 
We also allow for
settings with many covariates: 
identifiability 
and consistency still hold when 
the number of covariates is strictly larger than the number of instruments -- if the causal effect can assumed to be sparse (see Sections~\ref{sec:identifsparse} and~\ref{sec:estimation}). %

\subsection{Related work}
\paragraph{Two–sample IV and two-sample Mendelian randomization.}
Methods from the
two–sample IV and two-sample (summary-data) Mendelian randomization literature can be applied to our setting (even if they may not be consistent, see Section~\ref{sec:42}). 
Two-sample IV estimators combine instrument–exposure and instrument–outcome information from separate samples \citep{angrist1995split, inoue2010two, pacini2016robust, zhao2019two}; in 
econometrics,
such two-sample moment combinations have been used to study compulsory schooling and returns to education, for example in age-at-entry and split-sample IV 
designs \citep{AngristKrueger1992}.
Analogous designs underpin two–sample Mendelian randomization (MR) in genetics, where variants act as instruments \citep{hartwig2016two}.
Beyond these, two–sample MR theory highlights design issues such as winner’s-curse and sample overlap, which can bias naive estimators \citep{pierce2013efficient,burgess2016overlap}, as well as pervasive horizontal pleiotropy addressed by mode-based and mixture/likelihood approaches \citep{hartwig2017mode,qi2019mrmix,morrison2020cause}.
Our unpaired-sample formulation differs from previous work on two-sample IV and Mendelian randomization in that we allow for high-dimensional instruments and/or large number dimension of the treatment.
More concretely, we propose consistent estimators in regimes where the ratio of the sample size $n$ to the instrument dimension $m$ converges to a positive constant and/or the dimension of the treatment $d$ is larger than the dimension of the instrument $m$.

\paragraph{Sparse IV regression.}
A growing literature studies instrumental variable models in which the structural coefficient vector $\beta^*$ is sparse and the number of endogenous regressors and instruments can be large. One approach keeps the two-stage structure of \textsc{2SLS} while imposing an $\ell_1$-penalty in the structural equation. \citet{zhu2018sparse} analyzes $\ell_1$-regularized \textsc{2SLS} for triangular models, allowing both endogenous regressors and instruments to exceed the sample size, and establishes high-dimensional consistency and error bounds when the first-stage and structural parameters are sufficiently sparse. A related two-stage regularization framework is developed by \citet{lin2015regularization}, who use sparsity-inducing penalties in both stages to estimate fitted regressors and a sparse structural equation. In large systems of simultaneous equations, \citet{chen2018two} propose two-stage penalized least squares with ridge in the first stage and adaptive Lasso in the second, proving oracle-type results for support recovery and asymptotic normality of the selected nonzero effects. For inference in these settings, \citet{gold2020inference} construct a desparsified (one-step) GMM update for high-dimensional IV and show asymptotic normality of the debiased estimator, yielding valid confidence intervals for components of $\beta^*$ when initialized by two-stage Lasso-type estimators. Complementing these results, \citet{belloni2022manyendo} consider models with many endogenous variables and many instruments, using orthogonalized moments and Lasso only for nuisance estimation; they provide uniformly valid post-selection inference via multiplier bootstrap. 
More recently, \citet{pfister2022identifiability} establish identifiability of a sparse $\beta^*$ in linear IV models—which may be achievable when the instrument dimension is smaller than $d$—derive graphical conditions, and propose the \textsc{spaceIV} estimator 
(paired observations).
\citet{tang2023synthetic} also adopt the sparse causation premise and construct a synthetic instrument from $X$; under linear SEMs they show identifiability despite unmeasured confounding and cast estimation as $\ell_0$-penalization (paired observations).
Complementing these, \citet{huang2024sparse} study sparse causal effects with two-sample summary statistics (two-sample MR) when the variance of the instrument $\var(I)$ is invertible, adapting \textsc{spaceIV} to a summary-data regime. However, they do not provide theoretical guarantees such as consistency or asymptotic normality for their estimator.
A second strand regularizes IV estimation directly through penalized moment conditions, without an explicit two-stage Lasso second stage. \citet{caner2009lasso} introduces a Lasso-type GMM estimator that adds an $\ell_1$-style penalty to the GMM criterion and derives model selection and estimation guarantees under sparsity. Building on this idea, \citet{shi2016estimation} studies GMM-Lasso for linear structural models with many endogenous regressors and proves oracle inequalities for sparse $\beta^*$. In a related but distinct formulation, \citet{fan2014endogeneity} propose penalized focused GMM (FGMM) to address incidental endogeneity in high dimensions, establishing support recovery and asymptotic normality for the active coefficients under suitable conditions.

\paragraph{Many-instrument and weak-instruments.}
A large literature studies weak and many instruments in standard single-sample IV. We also consider the many instrument setting 
(or high-dimensional instrument)
setting. Since we assume that the 
norm 
of the instrument is constant with increasing sample size (which implies that the magnitude of each component of the instrument is decreasing in the sample size) our work also relates to the weak instrument literature. Early work by \citet{StaigerStock1997} showed that 
results in
conventional asymptotics can be misleading when instruments are weak, motivating alternative limiting frameworks and bias approximations. Many-instrument asymptotics (in the sense of the work by \citet{Bekker1994} and the subsequent contributions such as the ones by \citet{DonaldNewey2001} and \citet{HansenHausmanNewey2008})
analyze the behavior of IV estimators when the number of instruments grows with the sample size, and 
motivate
procedures that remain well behaved in that regime. Practical diagnostics and test procedures for weak instruments are surveyed by \citet{AndrewsStockSun2019}, who synthesize theory and practice for both weak and many instruments. \citet{StockYogo2005} develop widely used critical values for weak-IV tests based on first-stage $F$-statistics. High-dimensional extensions that allow for many controls and potential instruments include the post-double-selection methods of \citet{BelloniChernozhukovHansen2014}, which provide uniformly valid inference on treatment effects after model selection in sparse high-dimensional regressions, and related work on $\ell_1$-regularized IV and debiased machine learning for structural parameters \citep{BelloniChernozhukovHansen2012}.
\citet{ChoiGuShen2018} develop weak-instrument robust inference for two-sample IV, constructing Anderson--Rubin, Kleibergen--$K$, and conditional likelihood ratio-type tests whose validity does not rely on strong instruments and that remain robust to heterogeneous second moments across samples.
We analyze the many-instrument regime 
for unpaired observations,
a setting that appears to be largely unexplored in the existing literature. We show that naively constructed estimators in this regime are asymptotically biased due to an measurement-error effect that is specific to the two-sample, high-dimensional instrument (i.e., many-instrument) setting and does not arise in standard 
paired-sample
estimators. We show how to remove this bias and obtain consistent, asymptotically normal estimators.

Going beyond existing literature, we characterize identifiability for unpaired IV estimators in the regime 
$d>m$ (with a sparse structural parameter)
and  establish consistency and asymptotic normality of straight-forward estimators. 
Additionally, we 
characterize identifiability 
 when the instrument is high-dimensional ($m \to \infty$).
 For this setting, we propose novel estimators and  
prove consistency.

\section{Unpaired Data With Hidden Confounding}\label{sec:model}

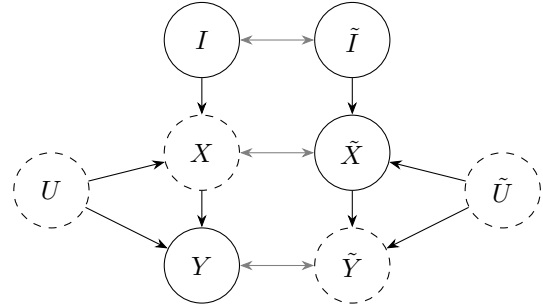
\begin{figure}[!htbp]
    \centering
    \begin{tikzpicture}[
          >=Stealth,
          node distance=1.2cm and 1.6cm,
          var/.style={draw, circle, inner sep=2pt, minimum width=10mm, minimum height=6mm},
          latent/.style={draw, circle, dashed, inner sep=2pt, minimum width=10mm, minimum height=6mm},
          err/.style={draw, circle, dashed, inner sep=1pt, minimum size=6mm}
        ]
        \node[var] (I) at (-1,2) {$I$};
        \node[var] (It) at (1,2) {$\ti I$};
        
        \node[latent] (H)   at (-3, 0) {$U$};
        \node[latent]    (X)   at (-1, 0.5) {$X$};
        \node[var]    (Y)   at (-1,-1) {$Y$};
        
        \node[latent] (Ht)  at ( 3, 0) {$\ti{U}$};
        \node[var]    (Xt)  at ( 1, 0.5) {$\ti{X}$};
        \node[latent]    (Yt)  at ( 1,-1) {$\ti{Y}$};
        
        \draw[->] (I) -- (X);
        \draw[->] (It) -- (Xt);
        
        \draw[->] (H)  -- (X);
        \draw[->] (X)  -- (Y);
        \draw[->] (H)  -- (Y);
        
        \draw[->] (Ht)  -- (Xt);
        \draw[->] (Xt)  -- (Yt);
        \draw[->] (Ht)  -- (Yt);

\draw[<->, gray] (X)  -- (Xt);
\draw[<->, gray] (I)  -- (It);
\draw[<->, gray] (Y)  -- (Yt);
    
    \end{tikzpicture}
    \caption{
    Simplified visualization of the data generating process in Eq.~\eqref{eq:scm}. $U$ and $\ti U$ represent unobserved confounders that visualize the dependence between $X$ and $\epsilon$. 
    The model allows for complex dependencies (and even equalities) between variables in the tilde and non-tilde world, as indicated by the gray edges.
    We discuss other, including more general data-generating processes that satisfy \Cref{eq:scm} in \Cref{sec:more_graphs}.
    }
    \label{fig:dag_intro}
\end{figure}

We consider two triplets of random variables: $(I, X, Y)$ and $(\ti I, \ti X, \ti Y)$ but assume access only to partial observations. Specifically, we observe $(I, Y)$ and $(\ti I, \ti X)$ (or, more precisely, multiple i.i.d.~copies thereof), the variables $X$ and $\ti Y$ remain unobserved. The corresponding data-generating process is given by (unobserved variables are shown in gray):
\begin{align}\label{eq:scm}
  Y = {\color{gray}{X^{\top}}} {\color{gray}{\beta^*}} + {\color{gray}{\epsilon}}, 
  \qquad
  {\color{gray}{\ti Y}} = \ti X^{\top} {\color{gray}{\ti \beta^*}} + {\color{gray}{\ti \epsilon}},
\end{align}
where $\beta^*$ and $\ti \beta^*$ represent the causal effects, that is,
    $\mathbb{E}[Y \mid \mathrm{do}(X := x)] = x^{\top} \beta^*$,
    $\mathbb{E}[\ti Y \mid \mathrm{do}(\ti X := x)] = x^{\top} \ti \beta^*$.

The noise terms $\epsilon$ and $\ti \epsilon$ are not assumed to be independent of $X$ and $\ti X$, respectively (thereby, implicitly, allowing for hidden confounding between $X$ and $Y$.
\Cref{fig:dag_intro} shows a graphical representation of the data-generating process. The setting even allows for dependence between the two samples. 
We discuss other data-generating processes 
that satisfy \Cref{eq:scm} in \Cref{sec:more_graphs}.

In Section~\ref{sec:identifsparse}, we show 
that the causal effect $\beta^*$ (or $\ti \beta^*$ or both) is identifiable. This is non-trivial
because of the following reasons: 
(i) Due to the hidden confounding, 
naively regressing $Y$ on $X$ (or $\ti Y$ on $\ti X$) (if these were available) would generally lead
 to biased estimates.
(ii) 
We only have unpaired data, so
regressing $Y$ on $X$ (or $\ti Y$ on $\ti X$) is not feasible.
Unlike classical work on unpaired data or two-sample IV $(I, X, Y)$ and $(\ti I, \ti X, \ti Y)$
may but do not have to be equal in distribution. 
(iii) 
Observations across the two samples may be dependent, see, e.g.,  \Cref{sec:more_graphs}.

\subsection{Setting and main assumptions} \label{sec:setting}
Let $n, \ti n, m, d \in \mathbb{N}$. Let $(I, X, Y)$ and $(\ti I, \ti X, \ti Y)$ be random vectors, 
whose components are $m$, $d$, and $1$-dimensional, respectively, and which
generated by the process described in \Cref{eq:scm} where $\beta^*, \ti \beta^* \in \mathbb{R}^d$ and $\epsilon$ and $\ti \epsilon$ are random variables.
We assume that all random variables have finite second moments.
We assume that there exist two i.i.d.~samples $\{(I_i,{\color{gray}{X_i}},Y_i)\}_{i\in[n]}$ and $\{(\ti I_i,\ti X_i,{\color{gray}{\ti Y_i}})\}_{i\in[\ti n]}$, but the observed data are \emph{unpaired}: in the first system we only observe $\{(I_i,Y_i)\}_{i\in[n]}$ and in the second $\{(\ti I_i,\ti X_i)\}_{i\in[\ti n]}$. 
The case $(I, X, Y) \overset{d}{=} (\ti I,\ti X,\ti Y)$ is contained as a special case.

\begin{assumption}
\label{ass:1}
    We have
    \begin{enumerate}[label=(\roman*), nosep]
      \item \label{ass:1:1} $\cov(\ti I, \ti X) = \cov(I, X)$ and
      \item \label{ass:1:2} $\mathbb{E}[\epsilon \mid I] = 0.$
    \end{enumerate}
\end{assumption}

{
\renewcommand\theassumption{1'}
\begin{assumption}
\label{ass:1'}
    We have
    \begin{enumerate}[label=(\roman*'), ref=(\roman*'), nosep]
        \item \label{ass:1':1} $\cov(\ti I, \ti Y) = \cov(I, Y)$ and
        \item \label{ass:1':2} $\mathbb{E}[\ti \epsilon \ |  \ti I] = 0$.
    \end{enumerate}
\end{assumption}
\renewcommand\theassumption{\arabic{assumption}}%
}

We will see in Section~\ref{sec:identifsparse} that under \Cref{ass:1} we can identify $\beta^*$ in the non-tilde system and under \Cref{ass:1'} we can identify $\ti \beta^*$ in the tilde system.
To simplify notation, in this work, we focus on identifying $\beta^*$ using \Cref{ass:1}; analogous arguments hold when we identify $\ti \beta^*$ using \Cref{ass:1'}.
\Cref{ass:1}~\ref{ass:1:2} (or \Cref{ass:1'}~\ref{ass:1':2}) is called exclusion restriction  or exogeneity and is commonly exploited in the IV literature \citep{angrist1995identification, angrist1996identification, angrist2009mostly}; it allows us to identify the causal effect despite the hidden confounding. 
\Cref{ass:1}~\ref{ass:1:1} (or \Cref{ass:1'}~\ref{ass:1':1}) connects the two systems and guarantees that we can transfer information from one system to the other.

\subsection{Environments as instruments}
In Section~\ref{sec:intro}, we have argued that if
we have access to data collected from multiple environments---for instance, from different experimental setups or hospitals---but we lack explicit features describing these environments 
we can still treat the environment indicator as a vector of categorical instruments (via one-hot encoding).
In such a setting, a sufficient condition for \Cref{ass:1}~\ref{ass:1:1} is that we observe $\ti X$ and $Y$ under the same environment distribution, 
and that the conditional expectations
of $X$ and $\ti X$ coincide within each environment: 
$\mathbb{E}[X | I] = \mathbb{E}[\ti X | I]]$.
An analogous condition holds for \Cref{ass:1}~\ref{ass:1':1}.

\section{Identifiability } \label{sec:identifsparse}
We now show that, 
under Assumption~\ref{ass:1},
the causal effect $\beta^*$ is identifiable from the joint distributions of $(\ti I, \ti{X})$ and 
$(I, Y)$. Without loss of generality and to simplify notation, we assume that all variables are centered and that all variables have finite second moments.
All proofs can be found in the appendix.

We first consider the partially known \emph{finite-dimensional instrument} setting (defined as the setting with $m$ fixed) in \Cref{sec:fd} and later, in \Cref{sec:hd}, we consider the \emph{high-dimensional instrument} setting (defined as the asymptotic setting with $m \to \infty$). 
In both settings, we consider $d$ to be fixed and differentiate the cases where $\beta^*$ is dense and where $\beta^*$ is sparse.
When considering consistency and inference, the settings correspond to  
$m$ fixed, $n, \ti n \to \infty$ (finite-dimensional instrument)
and 
$m, n \to \infty$ and $n/m \to r \in (0, \infty)$ (high-dimensional instrument), respectively.

Some %
remarks are in order. Although we rely on standard results for the Lasso \citep{tibshirani1996regression, buhlmann2011statistics}, the sparse $\beta^*$ setting is not `high-dimensional' in the the sense of $d \to \infty$ (even if the distribution of $(I, X, Y)$ is known, $\beta^*$ is, in general, not identified  if $d > \rank{\cov(I, X)}$).
In contrast, in the high-dimensional instrument setting,
$d$ is fixed but $m \rightarrow \infty$ (and 
$n/m \to r \in (0, \infty)$). In this case, we 
assume that the limit $Q \coloneqq \lim_{m\to\infty} m \cov(\ti I,\ti X)^{\top} \cov(\ti I,\ti X) \in \mathbb{R}^{d \times d}$ exists and is well-defined.

\subsection{Finite-Dimensional Instrument}
\label{sec:fd}
Under \Cref{ass:1}
we have
\begin{equation}
\label{eq:main}
\begin{aligned}
    &\cov(I, Y) -  \cov(\ti I, \ti X) \beta^*
    = \cov(I, Y) - \cov(I, X) \beta^*\\
&\quad     = \cov(I, Y -  X^{\top} \beta^*)
    = \mathbb{E}\left[I   \mathbb{E}\left[\epsilon\ | \ I\right]\right]
    = 0.
\end{aligned}
\end{equation}
Here, we used \Cref{ass:1}~\ref{ass:1:1} in the first equation and \Cref{ass:1}\ref{ass:1:2} in the last equation. Analogously, we can use \Cref{ass:1'} to get $\cov(I, Y) -  \cov(\ti I, \ti X) \ti \beta^* = 0$. 
We now define the set of solutions $\mathcal{S}$ as
\begin{align*}
    \mathcal{S} &\coloneqq \{ \beta \in \mathbb{R}^d \ |  \cov(I, Y) =  \cov(\ti I, \ti X) \beta\}
\end{align*}
($\mathcal{S}$ can be computed from the joint distribution of $(\ti{X}, \ti I)$ and the joint distribution of $(I, Y)$).
We always have $\beta^* \in \mathcal{S}$.
The causal effect $\beta^* \in \mathbb{R}^d$ is identifiable from the joint distribution of $(\ti{X}, \ti I)$ and the joint distribution of $(I, Y)$ via these moment-conditions if and only if $\mathcal{S} = \{\beta^*\}$.
We now state a well-known result from the two-sample IV (also called two-sample MR) literature
\citep{angrist1995split, inoue2010two, hartwig2016two}.
\begin{proposition}[Identifiability for dense $\beta^*$]
\label{thm:iden_dense}
    Assume \Cref{ass:1}. 
    We have $\mathcal{S} = \{\beta^*\}$ if and only if
    \begin{equation}
    \label{eq:3453}
        \mathrm{rank}(\cov(I, X)) = d.
    \end{equation}
    In particular, a necessary condition for identifiability is $m \geq d$.
\end{proposition}
\Cref{thm:iden_dense} shows that we have identifiability in our general unpaired setting under the same conditions as for the paired IV setting.

Next, we consider the sparse setting, i.e., 
$\|\beta^*\|_0 \leq s^* \leq d$. 
In this case, we may have identifiability even if 
$\mathrm{rank}(\cov(I, X)) < d$.
We now consider
\begin{equation*}
    \mathcal{S}_{0} \coloneqq \argmin_{\beta \in \mathcal{S}} \|\beta\|_0.
\end{equation*}
Defining $\Sigma_t\coloneqq\{v\in\mathbb{R}^d:\|v\|_0\leq t\}$, we obtain the following identifiability result.
\begin{theorem}[Identifiability for sparse $\beta^*$]
\label{thm:sparse_ident}
    Assume \Cref{ass:1}. The following statements are equivalent:
    \begin{enumerate}[label=(\roman*), nosep]
        \item (\emph{Identifiability via sparsest solution})
        For all $\beta^* \in \mathbb{R}^d$ with
        $\|\beta^*\|_0 \le s^*$, it holds that
        \begin{equation*}
            \mathcal{S}_0 = \{\beta^*\}.
        \end{equation*}
        \item \label{thm:sparse_ident:2}(\emph{Restricted nullspace}) $\ker \left( \cov(I, X) \right)\cap \Sigma_{2s^*}=\{0\}$.
    \end{enumerate}
    In particular, a necessary condition for identifiability is $m \geq 2s^* $.
\end{theorem}
\Cref{thm:sparse_ident} shows that the restricted nullspace condition for $\cov(I, X)$ is a necessary and sufficient condition for identifiability via $\mathcal{S}_0$. The conditions for sparse identifiability are weaker than identifiability in the dense setting and they are strictly weaker if and only if $d \geq 2s^*$. Intuitively, because we want uniqueness among all $s^*$-sparse solutions, we need to ensure that the difference of two $s^*$-sparse vectors lies in the nullspace, and that difference can have up to $2s^*$ nonzeros, hence the $2$.
This shows that in settings where we can reasonably assume sparse causal effects in unpaired data, 
identifiability is still possible even with weak instruments (e.g., when $m$ is small).

\subsection{High-Dimensional Instruments}
\label{sec:hd}
In the high-dimensional instrument setting we assume that $m \to \infty$, 
so the dimensions of $\cov(I, Y)$ and $\cov(\ti I, \ti X)$ increase with $n$. We consider the limit $Q = \lim_{m\to\infty} m \cov(\ti I,\ti X)^\top \cov(\ti I,\ti X)$ assuming that it exists; 
conditions for identifiability are then expressed in terms of $Q$ rather than the finite-$m$ matrix \footnote{This setting in quite natural when considering categorical instruments as it ensures that $\var(X)$ stays bounded as $m \to \infty$. For details see \Cref{ex:ci}.}. 
Under \Cref{ass:1}, we have $Q_Y \coloneqq \lim_{m\to\infty} m \cov(\ti I,\ti X)^{\top} \cov(I,Y) = Q\beta^*$:
\begin{align*}
&    m\cov(\ti I,\ti X)^{\top} \cov(I,Y) \\
=\;  &m\cov(\ti I,\ti X)^{\top} \big(\cov(I,X) \beta^* + \cov( I, \epsilon) \big)\\
    =\;  &m\cov(\ti I,\ti X)^{\top} \cov(\ti I,\ti X) \beta^*
     \to Q \beta^*.
\end{align*}
Thus,
\begin{equation*}
    \beta^* \in \mathcal{S}_{\infty} \coloneqq \{ \beta \in \mathbb{R}^d \ | \ Q \beta = Q_Y\}.
\end{equation*}

\begin{theorem}[Identifiability for high-dimensional instrument and dense $\beta^*$]
\label{thm:iden_dense_hd}
    Assume \Cref{ass:1}. We have $\mathcal{S}_{\infty} = \{\beta^*\}$ if and only if
    \begin{equation*}
        \mathrm{rank}(Q) = d.
    \end{equation*}
\end{theorem}
Defining
$\mathcal{S}_{0, \infty} \coloneqq \argmin_{\beta \in \mathcal{S}} \|\beta\|_0
$, we also prove the result for sparse $\beta^*$. 
\begin{theorem}[Identifiability for high-dimensional instrument and sparse $\beta^*$]
\label{thm:sparse_ident_hd}
    Assume \Cref{ass:1}. Fix $s^* \leq d$. The following statements are equivalent:
    \begin{enumerate}[label=(\roman*), nosep]
        \item (\emph{Identifiability via sparsest solution})
        For all $\beta^* \in \mathbb{R}^d$ with
        $\|\beta^*\|_0 \le s^*$, it holds that
        \begin{equation*}
            \mathcal{S}_{0, \infty} = \{\beta^*\}.
        \end{equation*}
        \item (\emph{Restricted nullspace}) $\ker \left( Q \right)\cap \Sigma_{2s^*}=\{0\}$.
    \end{enumerate}
\end{theorem}

\section{Estimation 
}\label{sec:estimation}
For the setting with finite-dimensional instrument and sparse $\beta^*$ we propose a GMM estimator with $\ell_1$ regularization.
For the high-dimensional instrument setting, we propose novel GMM-based estimators (for dense and sparse $\beta^*$) and show that they are consistent. Moreover, we demonstrate that in the high-dimensional instrument regime the standard two-sample IV estimator is asymptotically biased, and we introduce a cross-moment GMM estimator that remains consistent and asymptotically normal.

As in 
\Cref{sec:setting} we assume that $\{(I_i,Y_i)\}_{i=1}^n$ are i.i.d.\ and  $\{(\ti I_j,\ti X_j)\}_{j=1}^{\ti n}$ are i.i.d., but, additionally, also that the two samples are independent. %
Define
$
     N \coloneqq n + \ti n, \tau_n \coloneqq \frac{n}{N} \to \tau \in (0, 1),
     \ti\tau_n \coloneqq \frac{\ti n}{N} \to \ti \tau \in (0, 1).
$

\subsection{Finite-Dimensional Instrument}
For completeness and to help with intuition, we first state the consistency and asymptotic normality for dense $\beta^*$. 
Similar results have appeared in the standard two-sample IV literature \citep{angrist1991compulsory, inoue2010two, burgess2013summ, bowden2015egger}.

For a positive definite weighting matrix $W_N\in\mathbb R^{m\times m}$, define the sample moment
\begin{equation}\label{eq:sample-moment}
    \hat g_N(\beta) \coloneqq \frac{1}{n}\sum_{i=1}^n I_i Y_i - \Big(\frac{1}{\ti n}\sum_{j=1}^{\ti n} \ti I_j \ti X_j^{\top}\Big)\beta \in \mathbb R^m,
\end{equation}
and the \estnfd{} estimator
\begin{align}\label{eq:betaGMM-explicit}
    \gmm(W_N) &\coloneqq \arg\min_{\beta\in\mathbb R^d}  \hat g_N(\beta)^{\top} W_N \hat g_N(\beta).
\end{align}
If we choose $W_N = \mathrm{Id}$ then \estnfd{} coincides with TS-IV. Define
$
    \Omega \coloneqq \tau^{-1}\Omega_m + \ti\tau^{-1}\Omega_c(\beta^*),\Omega_m \coloneqq \operatorname{Var}(IY), \Omega_c(\beta^*) \coloneqq \operatorname{Var}\big(\ti I \ti X^{\top}\beta^*\big),
$
and denote by $\hat{\Omega}$ the sample version of $\Omega$ (see \Cref{sec:dense} for details). The asymptotic variance of $\gmm$ is minimized choosing $W_N$ to be equal $\widehat W \coloneqq \hat{\Omega}^{-1}$ in \eqref{eq:betaGMM-explicit}.
The moment variance reflects the two-sample structure: the $IY$ term scales with 
$
    \tau^{-1} = \lim_{n \to \infty} N /n
$
and the $\ti I \ti X^{\top}\beta^*$ term with
$
    \ti\tau^{-1} = \lim_{\ti n \to \infty} N /\ti n.
$
For details, see \Cref{sec:dense}.
\begin{proposition}[informal version of \Cref{thm:denseAN}]
\label{thm:denseAN-informal}
    Under some technical assumptions and if $\rank{\cov(I, X)} = d$ we have that
    \begin{equation}
        \gmm(W_N) \overset{p}{\to} \beta^*.
    \end{equation}
    Furthermore, $\gmm(W_N)$ is asymptotically normal.
\end{proposition}

\paragraph{Sparse $\beta^*$ and penalized GMM.}
When the full-rank condition in \Cref{thm:denseAN-informal} fails, dense identification is in general not possible:
there are too few independent moment conditions to recover an arbitrary $d$-dimensional $\beta^*$.
We instead impose sparsity and require only that the moment operator be well conditioned on sparse cones (a restricted-eigenvalue/compatibility condition).
Then the $\ell_1$-penalized GMM in \eqref{eq:LassoGMM-explicit} consistently estimates $\beta^*$, attains the rates with $\lambda \asymp \sqrt{1/N}$ , and recovers the support under a beta-min condition (see \Cref{thm:ratesSparse}). 

More formally, when 
$
    \|\beta^*\|_0=s^*\leq d,
$
we estimate $\beta^*$ by the $\ell_1$–penalized \estnfd{}
\begin{equation}\label{eq:LassoGMM-explicit}
    \ivs(W_N) \in \arg\min_{\beta\in\mathbb R^d}  \frac{1}{2}\Big\| W_N^{1/2}\hat g_N(\beta)\Big\|_2^2 + \lambda_N \|\beta\|_1,
\end{equation}
where $\hat g_N(\beta)$ is given by \eqref{eq:sample-moment} and $\lambda_N \asymp 1/\sqrt{N}$ \footnote{In practice, one often uses cross-validation or an information criterion to choose $\lambda$.}.

\begin{assumption}
\label{ass:RE}
\begin{enumerate}[label=(\roman*), nosep]
    \item \label{ass:RE:0} \Cref{ass:1} holds, the centering convention \footnote{We center random variables by subtracting their sample mean.} and boundedness of fourth moments.
    \item \label{ass:RE:3} There is $W_0$ positive definite such that $W_N \overset{p}{\to} W_0$.
    \item \label{ass:RE:1} (Restricted eigenvalue.) There exists $\kappa>0$ such that for all $S \subseteq [d]$ with $|S| = s^*$,\footnote{We define $\|a\|_W \coloneqq a^T W a$.}
    \begin{equation}
        \kappa \leq \inf_{\Delta : \|\Delta_{S^c}\|_1 \leq 3 \|\Delta_{S}\|_1 } \frac{\|\cov(\ti I,\ti X)\Delta \|_{W_0}^2}{ \|\Delta\|_2^2}.
    \end{equation}
\end{enumerate}
\end{assumption}
Note that \Cref{ass:RE}~\ref{ass:RE:1} implies \Cref{thm:sparse_ident}~\ref{thm:sparse_ident:2} and therefore \Cref{ass:RE}~\ref{ass:RE:1} is at least as strong as \Cref{thm:sparse_ident}~(ii).
\Cref{ass:RE}~\ref{ass:RE:1} is a standard assumption in the Lasso literature (see, e.g., \citet{buhlmann2011statistics}).

We can now prove consistency and support recovery.
Let $S^* := \mathrm{supp}(\beta^*)$
denote the true support of $\beta^*$. 
\begin{theorem}[Rates for penalized GMM]
\label{thm:ratesSparse}
    For all $\beta^*$ with $\|\beta^*\|_0 = s^*$, under \Cref{ass:RE} and $\lambda_N \asymp 1/\sqrt{N},$ the estimator \eqref{eq:LassoGMM-explicit} satisfies
    \begin{equation*}
        \|\ivs(W_N)-\beta^*\|_2 = O_{p}\left(\sqrt{s^* /N}\right).
    \end{equation*}
    In particular, $\ivs(W_N) \overset{p}{\to} \beta^*$.
    If, in addition, there exists $c>0$ such that for all $i \in S^*$,
    we have $|\beta^*_i| \geq c$ (beta-min condition), then
    \begin{equation*}
        \widehat S (W_N) \coloneqq \{ j \in [d] \ | \ |\ivs(W_N)_j| \geq c/2\} \overset{p}{\to} S^*.
    \end{equation*}
\end{theorem}

\paragraph{Inference on the estimated support.}
The Lasso is used only for support recovery. Under beta–min and regularity conditions, $\widehat S$ is consistent: $\mathbb{P}(\widehat S=S^*)\to1$. We then refit an unpenalized GMM on $\widehat S$ to obtain $\ivsf$, reducing to a standard finite-dimensional GMM problem with asymptotic normality and sandwich standard errors \citep{hansen1982gmm}. Since mis-selection becomes negligible under a beta-min condition,
usual Wald intervals based on the refitted GMM standard errors are asymptotically valid. We make this precise in \Cref{thm:oracle-CI}.

\begin{theorem}[Oracle CIs on the estimated support]
\label{thm:oracle-CI}
    Assume \Cref{ass:RE} and the beta-min condition from \Cref{thm:ratesSparse} for $W_N$ and let $\hat{S}(W_N)$ be defined as in \Cref{thm:ratesSparse}. Assume that there is a second sequence of weight matrices $W'_N \overset{p}{\to} W'_0 \succ 0$. Define
    \begin{align*}
        \ivsf(W'_N) &\coloneqq \argmin_{\beta\in \mathbb{R}^d : \beta_{\widehat S^c (W_N)} = 0}\ (\widehat{\cov}(I, Y)-\widehat{\cov}(\ti I, \ti X)\beta)^{\top} \\
        &\qquad \qquad \qquad \cdot W'_N (\widehat{\cov}(I, Y)-\widehat{\cov}(\ti I, \ti X)\beta),
    \end{align*}
    and
    \begin{align*}
        V_{S^*} := &(\cov(I, X)_{S^*}^{\top} W'_0 \cov(I, X)_{S^*})^{-1}\\
        &\cdot \big(\cov(I, X)_{S^*}^{\top} W'_0 \Omega W'_0 \cov(I, X)_{S^*}\big)  \\
        &\cdot (\cov(I, X)_{S^*}^{\top} W'_0 \cov(I, X)_{S^*})^{-1},
    \end{align*}
    where $\Omega \coloneqq \var(\sqrt{N} \hat{g}_N(\beta^*))$.
    We have that
    \begin{equation*}
        \sqrt{N}\big(\ivsf(W'_N)-\beta^*\big) \overset{d}{\to} \mathcal N \big(0, \ti{V}\big),
    \end{equation*}
    where $\ti{V}$ has the $S^* \times S^*$ block equal to $V_{S^*}$ and is zero elsewhere.
\end{theorem}

\begin{remark}
    \label{rem:34}
    The asymptotic variance $V_{S^*}$ only depends on $W'_N$ and not on $W_N$. We are therefore free to choose $W_N = \mathrm{Id}_m$.  We define $W'_N \coloneqq \hat{\Omega}_{\hat{S}(\mathrm{Id}_m)}^{-1} \in \mathbb{R}^{m \times m}$, where $\hat{\Omega}_{\hat{S}(\mathrm{Id}_m)}$ is defined as in \Cref{eq:OmegaHat-explicit} but we consider all variables restricted to $\hat{S}(\mathrm{Id}_m)$.
    If $\hat{\Omega}_{\hat{S}(\mathrm{Id}_m)}$ is not positive definite, we choose $W'_N \coloneqq \mathrm{Id}_m$ (we do this to ensure that the estimator is well-defined). Since $\hat{S}(\mathrm{Id}_m) \overset{p}{\to} S^*$ we have $W'_N \overset{p}{\to} \Omega_{S^*}^{-1}$ and therefore this choice minimizes the asymptotic variance.
\end{remark}

\subsection{High-Dimensional Instrument} \label{sec:42}
We first show that the naive TS-IV estimator is asymptotically biased in the high-dimensional instrument setting. This bias is a consequence of measurement error in the estimate of $\cov (I, X)$ and appears only in two-sample IV. It is therefore distinct from the endogeneity problem that biases the one-sample IV. The measurement error in high-dimensional unpaired IV is related to the attenuation bias in the paired weak-instrument setting \citep{angrist1995split, ChoiGuShen2018}. In the weak-instrument literature attenuation appears as a result of removing the  weak-instrument bias with the split-sample IV estimator which as a consequence introduces attenuation. We show that high-dimensional instrument unpaired IV suffer from a similar bias and prove that sample-splitting is an effective solution for the measurement-error bias.

\paragraph{Classical estimators are biased
in paired and unpaired settings.}

Assume \Cref{ass:1}. For notational simplicity we consider $\ti n = n=r m$ (for some $r \in \mathbb{N}$) , $d=1$ and $m \to \infty$. We show that the naive two-sample IV estimator is asymptotically biased for general $d \in \mathbb{N}$, $r = n/m \geq  1$ and $\ti r = \ti{n} /m \geq  1$.
The naive estimator solves \Cref{eq:betaGMM-explicit} with $W_N = Id$, that is,
\begin{align} 
    \hat \beta &\coloneqq \frac{\widehat{\cov}(\ti I, \ti X)^{\top} \widehat{\cov}(I, Y)}{\widehat{\cov}(\ti I, \ti X)^{\top} \widehat{\cov}(\ti I, \ti X)}  \nonumber \\
    &= \beta^* \frac{\widehat{\cov}(\ti I, \ti X)^{\top} \widehat{\cov}(I, X) }{\widehat{\cov}(\ti I, \ti X)^{\top} \widehat{\cov}(\ti I, \ti X)} + \frac{\widehat{\cov}(\ti I, \ti X)^{\top} \widehat{\cov}(I, \epsilon) }{\widehat{\cov}(\ti I, \ti X)^{\top} \widehat{\cov}(\ti I, \ti X)}.
    \label{eq:biasedunpaired}
\end{align}
If $m$ were constant and $n \to \infty$, then $\widehat{\cov}(\ti I, \ti X) \overset{p}{\to} \cov(I, X)$, $\widehat{\cov}(I, X) \overset{p}{\to} \cov(I, X)$ and $\widehat{\cov}(I, \epsilon) \overset{p}{\to} 0$ and therefore $\hat \beta \overset{p}{\to} \beta^*$ (and independent of the fact that $\widehat{\cov}(I, X)$ cannot be computed from the data). However, in the high-dimensional case where $\ti n/m \to \ti r \in (0, \infty)$ in general it does not hold that $ \|\widehat{\cov} (\ti I, \ti X) - \cov(I, X)\|_2^2 \overset{p}{\to} 0$ resulting in an inconsistent estimator.

\begin{lemma}
\label{lem:inconst}
    Assume \Cref{ass:hd-weak} with d=1. Additionally, assume that $\tr(\Sigma_{IX})\to b\in(0,\infty)$ and $\sup_m \mathbb{E} [\|IX - c\|_2^4] \leq \infty$. Then 
    \begin{equation*}
        \hat \beta \overset{p}{\to} \beta^* \frac{Q}{Q + \frac{b}{\ti r}} \neq \beta^*.
    \end{equation*}
\end{lemma}

\Cref{lem:inconst} shows that in the finite-dimensional instrument setting ($m$ is constant and $n \to \infty$) we have $r = n/m \to \infty$ and we recover consistency; however, in the high-dimensional instrument case $r$ is a constant.
In conclusion, unpaired IV is asymptotically biased.
In contrast, in the paired high-dimensional instrument setting, \textsc{2SLS} is known to be biased as well, but for a different reason: here,  $\widehat{\cov}(I, \epsilon) \not \to 0$. This is often called many-instrument bias, and methods such as LIML \citep{anderson1949estimation, fuller1977some} or SS-IV \cite{angrist1995split}
have been developed to remove the bias.

A simple solution to the measurement-error problem in unpaired IV is to use sample splitting in the denominator, i.e., we divide the $(\ti I, \ti X)$-data into two samples $A \subseteq [\ti n]$ and $B = A^c$ of equal size. Then define $E \coloneqq \widehat{\cov} (\ti I, \ti X) - \cov(I, X)$, $E_A \coloneqq \widehat{\cov}_A (\ti I, \ti X) - \cov(I, X)$ and $E_B \coloneqq \widehat{\cov}_B (\ti I, \ti X) - \cov(I, X)$ (where $\widehat{\cov}_A (\ti I, \ti X)$ is estimated on fold $A$ and $\widehat{\cov}_B (\ti I, \ti X)$ on fold $B$). In this case we have $E_A^{\top}E_B$ instead of $E^TE$ in the denominator and because $E_A \independent E_B$ we get $E_A^{\top} E_B \overset{p}{\to} 0$ recovering consistency in the high-dimensional setting. We now make this idea precise.

\paragraph{Unbiased estimators via cross moments.}
Split the $(\ti I,\ti X)$-sample into $K\ge 2$ disjoint, equal-sized folds and let
\begin{equation*}
    \widehat{\cov}_k(\ti I,\ti X),\qquad k\in [K],
\end{equation*}
be the foldwise covariance vectors (computed only within fold $k$).
Define
\begin{align*}
    C_{XX} &\coloneqq \frac{m}{K(K-1)}\sum_{k\neq h}\widehat{\cov}_h(\ti I,\ti X)^{\top} \widehat{\cov}_k(\ti I,\ti X),\\
    C_{XY} &\coloneqq m\widehat{\cov}(\ti I,\ti X)^{\top} \widehat{\cov}(I,Y).
\end{align*}
and
\begin{align*}
    \hat{g}(\beta) &\coloneqq C_{XY} - C_{XX} \beta\\
    &= \frac{m}{K} \sum_{k=1}^K \widehat{\cov}_k(\ti I, \ti X)^{\top} \\
    & \quad\left( \widehat{\cov}(I, Y) - \frac{1}{K-1} \sum_{h \neq k} \widehat{\cov}_h(\ti I, \ti X) \beta \right).
\end{align*}
Define the \estn{} estimator
\begin{align*}
    \hat\beta^{\mathrm{UP}, \mathrm{HD}}_{\mathrm{GMM}} (W_N)
    &\coloneqq
    \arg\min_{\beta} \hat g(\beta)^{\top}  W_N^{-1}  \hat g(\beta)\\
    &= 
    \big(C_{XX}^{\top}W_N^{-1} C_{XX}\big)^{-1} C_{XX}^{\top}W_N^{-1} C_{XY},
\end{align*}
where $W_N$ is a weight matrix such that $W_N \to W_0$ as $m \to \infty$ and where $W_0$ is positive definite.
The key property is that $C_{XX}$ is built from \emph{independent} views of $\cov(I, X)$, so $\mathbb{E}[C_{XX}]=m\cov(I, X)\cov(I, X)^{\top}$ without measurement-error.

\begin{assumption}
    \label{ass:hd-weak}
    Assume $d\in\mathbb{N}$ and $K\ge 2$ are fixed,
    and $n/m\to r$, $\ti n/m\to \ti r$ as $m\to\infty$. Additionally, assume all of the following.
    \begin{enumerate}[label=(\roman*), nosep]
        \item \label{ass:hd-weak:23}
        \Cref{ass:1} holds, we use the centering convention and boundedness of fourth moments.

        \item \label{ass:hd-weak:1}
        The limit
        \begin{equation*}
            Q \coloneqq \lim_{m\to\infty} m\cov(\ti I,\ti X)^{\top}\cov(\ti I,\ti X)
            \in\mathbb{R}^{d\times d}
        \end{equation*}
        exists and is positive definite.

        \item \label{ass:hd-weak:3}
        Writing $\Sigma_{IX}\coloneqq \var(\mathrm{vec}(\ti I \ti X))\in\mathbb{R}^{md\times md}$
        and $\Sigma_{IY}\coloneqq \var(IY)\in\mathbb{R}^{m\times m}$, there exists
        $C<\infty$ such that uniformly for all $m\in\mathbb{N}$
        \begin{equation*}
            m\|\Sigma_{IX}\|_{\mathrm{op}}\le C
            \quad\text{and}\quad
            m\|\Sigma_{IY}\|_{\mathrm{op}}\le C.
        \end{equation*}
    \end{enumerate}
\end{assumption}

\Cref{ass:hd-weak}~\ref{ass:hd-weak:1} 
replaces \Cref{eq:3453} from the low-dimensional case.
\Cref{ass:hd-weak}~\ref{ass:hd-weak:3} imposes a uniform bound on the operator norms of the covariance of $IX$ and $IY$, meaning that the cross–sectional noise in the first stage and reduced form does not explode as the number of instruments grows.

\begin{example}[categorical instruments]
\label{ex:ci}
    Let $K\sim\mathrm{Unif}\{1,\dots,m\}$, $\bar I=e_K$, and $I=\bar I-\mathbb E[\bar I]=e_K-\frac1m\mathbbm 1$. Consider a categorical first stage
    \begin{equation*}
        X=\mu^\top \bar I+\epsilon=\mu_K+\epsilon,\qquad \mathbb E[\epsilon\mid K]=0,
    \end{equation*}
    and assume uniformly bounded $8$th moments and
    \begin{equation*}
        \var(\mu^{\top} \ti I) = \frac1m\|\mu-\bar\mu\mathbbm 1\|_2^2\to Q\in(0,\infty).
    \end{equation*}
    Then
    \begin{align*}
        \cov(I,X)=\mathbb E[IX]&=\frac1m(\mu-\bar\mu\mathbbm 1),\\
        m\|\cov(I,X)\|_2^2&=\frac1m\|\mu-\bar\mu\mathbbm 1\|_2^2\to Q,
    \end{align*}
    so \Cref{ass:hd-weak}~\ref{ass:hd-weak:1} holds (here $d=1$).
    For \Cref{ass:hd-weak}~\ref{ass:hd-weak:3}, write $Z=IX\in\mathbb R^m$. For each coordinate $i$,
    \begin{equation*}
    Z_i=(\mathbbm 1\{K=i\}-1/m)X,
    \end{equation*}
    so $\var(Z_i)=O(1/m)$ and for $i\neq j$ we have $|\cov(Z_i,Z_j)|=O(1/m^2)$. Therefore each row sum of $\Sigma_{IX}=\var(IX)$ is $O(1/m)$, implying
    \begin{equation*}
        \|\Sigma_{IX}\|_{\mathrm{op}}\le \|\Sigma_{IX}\|_1=O(1/m),
    \end{equation*}
    hence
    \begin{equation*}
        m\|\Sigma_{IX}\|_{\mathrm{op}}=O(1)
    \end{equation*}
    uniformly in $m$. The same argument applies to $\Sigma_{IY}=\var(IY)$ (under bounded second moments of $Y$), giving $m\|\Sigma_{IY}\|_{\mathrm{op}}=O(1)$. Finally,
    \begin{equation*}
    \tr(\Sigma_{IX})=\sum_{i=1}^m \var(Z_i)\to b\in(0,\infty)
    \end{equation*}
    provided $\mathbb E[X^2]$ stays bounded away from $0$.
\end{example}

\begin{theorem}[Consistency]
\label{thm:cf-gmm-weak}
    Assume \Cref{ass:hd-weak}.
    Let $d\in\mathbb{N}$ and $K\ge 2$ be fixed, and $n/m\to r \in (0, \infty)$, $\ti n/m\to\ti r \in (0, \infty)$.
    Let $W_N \in \mathbb{R}^{d \times d}$ be a sequence of positive definite weight matrices such that
    $W_N\to W_0$ in probability with $W_0\succ 0$.
    Then
    \begin{equation*}
        \hat\beta^{\mathrm{UP},\mathrm{HD}}_{\mathrm{GMM}}(W_N)\overset{p}{\to}\beta^*.
    \end{equation*}
\end{theorem}

\subsubsection{High-dimensional instrument with sparse \texorpdfstring{$\beta^*$: $\ell_1$}{s}–regularized cross–moment GMM}

We assume that $n=rm$ and $\ti n=\ti r m$ for the constants $r,\ti r\in(0,\infty)$, while the causal effect is sparse with $s^*\coloneqq\|\beta^*\|_0\ll d$. Define the $\ell_1$–penalized estimator
\begin{align*}
    \hat\beta^{\mathrm{UP}, \mathrm{HD}}_{\mathrm{GMM}, \ell_1}(W_m)\in\arg\min_{\beta\in\mathbb R^d}
    &\frac{1}{2}\big\|W_m^{1/2}\big(C_{XY}-C_{XX}\beta\big)\big\|_2^2\\
    &+\lambda_m\|\beta\|_1.
\end{align*}
We take
$
    \lambda_m \asymp 1/\sqrt{{m}}.
$\footnote{In practice, one often uses cross-validation or an information criterion to choose $\lambda$.}

\begin{assumption}
\label{ass:RE-large-m-weak}
Assume all of the following.
\begin{enumerate}[label=(\roman*), nosep]
    \item \Cref{ass:1}, the centering convention and with bounded fourth moments.
    \item \label{ass:RE-large-m-weak:1} The limit $Q \coloneqq \lim_{m\to\infty} m \cov(\ti I,\ti X)^{\top} \cov(\ti I,\ti X)$ exists  and is well-defined.
    
    \item \label{ass:RE-large-m-weak:1.5} There is $W_0$ positive definite such that $W_m \overset{p}{\to} W_0$.
    
    \item \label{ass:RE-large-m-weak:2} There exists $\kappa>0$ such that for all supports $S\subset[d]$ with $|S|=s^*$ we have
    \begin{equation*}
        \kappa \leq \inf_{\Delta: \|\Delta_{S^c}\|_1\le 3\|\Delta_S\|_1}\frac{\Delta^{\top}\big(Q ^{\top}W_0Q \big)\Delta}{ \|\Delta\|_2^2}.
    \end{equation*}

    \item \label{ass:RE-large-m-weak:4} \Cref{ass:hd-weak}~\ref{ass:hd-weak:3} holds.
\end{enumerate}
\end{assumption}

\Cref{ass:RE-large-m-weak}~\ref{ass:RE-large-m-weak:2} is a standard assumption in the Lasso literature (see, e.g., \citet{buhlmann2011statistics}).

\begin{theorem}\label{thm:rates-large-m-sparse-weak}
    Let $S^*\coloneqq\mathrm{supp}(\beta^*)$ with $|S^*|=s^*$. Under \Cref{ass:RE-large-m-weak} and $\lambda_m\asymp\sqrt{1/m}$, the estimator $\hat\beta^{\mathrm{UP}, \mathrm{HD}}_{\mathrm{GMM}, \ell_1}$ satisfies
    \begin{equation*}
        \|\hat\beta^{\mathrm{UP}, \mathrm{HD}}_{\mathrm{GMM}, \ell_1}(W_m)-\beta^*\|_2=O_p\Big(\sqrt{\tfrac{s^*}{m}}\Big).
    \end{equation*}
    If, in addition, there exists $c>0$ such that $\min_{j\in S^*}|\beta^*_j|\geq c$ (beta–min), then 
    \begin{equation*}
        \widehat S(W_m) \coloneqq \{ j \in [d] \ | \ |\hat\beta^{\mathrm{UP}, \mathrm{HD}}_{\mathrm{GMM}, \ell_1}(W_m)|_j \geq c/2\} \overset{p}{\to} S^*.
    \end{equation*}
\end{theorem}

\subsection{
Variance Reduction in Practice}
\label{sec:vr}
In finite samples, cross-fitting stabilizes the denominator by removing the leading measurement-error bias, but it can introduce additional variability through random splitting. We therefore use two simple variance-reduction devices that preserve consistency and leave the preceding theory unchanged. All methods described in this section are detailed in \Cref{sec:alg_details}.

\paragraph{Monte Carlo averaging over random splits.}
To reduce the finite-sample variance introduced by sample splitting, we repeat the random split multiple times and average the resulting estimates of $C_{XX}$. Each split-specific estimator is consistent for the same population target, hence their average remains consistent. Consequently, all results from the previous subsections apply verbatim to the averaged estimator.

\paragraph{Closed form for the infinite-split average.}
A natural question is whether the split-average admits a closed form when we average over \emph{all} possible splits. The answer is yes. We fix the $(\tilde I,\tilde X)$ sample and consider $K=2$ splits into two equal halves. The cross-fit denominator for a split $(A,B)$ is
\begin{equation*}
    C_{XX}(A,B)\coloneqq m\,\widehat{\cov}_A(\tilde I,\tilde X)^{\top}\widehat{\cov}_B(\tilde I,\tilde X).
\end{equation*}
If we draw $H\in\mathbb{N}$ i.i.d.\ random splits $(A_1,B_1),\dots,(A_H,B_H)$ and let $H\to\infty$, the Monte Carlo average converges (conditional on the data) to the conditional expectation over a uniform random split. This expectation has the closed form
\begin{align*}
    \frac{1}{m}\,\bar C_{XX}
    &\coloneqq
    \lim_{H\to\infty}\frac{1}{H}\sum_{h=1}^H C_{XX}(A_h,B_h)\\
    &=
    \frac{1}{n(n-1)}\sum_{i\neq j} (\tilde I_i \tilde X_i^{\top})^{\top}(\tilde I_j \tilde X_j^{\top})\\
    &=
    \frac{n}{n-1}\widehat{\cov}(\tilde I,\tilde X)^{\top}\widehat{\cov}(\tilde I,\tilde X) \\
    &\quad -
    \frac{1}{n(n-1)}\sum_{i=1}^n (\tilde I_i \tilde X_i^{\top})^{\top}(\tilde I_i \tilde X_i^{\top}).
\end{align*}
Thus, the $H\to\infty$ split-average equals the usual plug-in quadratic form
$\widehat{\cov}(\tilde I,\tilde X)^{\top}\widehat{\cov}(\tilde I,\tilde X)$
minus a diagonal `self-inner-product' correction. This correction is  what removes the measurement-error bias present in the denominator in \Cref{eq:biasedunpaired}. The proof is given in \Cref{sec:proof_splits}.

\section{Experiments}
\label{sec:experiments}
We compare \estn{} and \estnfd{} against standard two-sample baselines (TS-IV and TS-2SLS) on synthetic and real-world data. Implementation details, hyperparameters, and data-generation specifics are deferred to \Cref{sec:dg_details}, additional experiments to \Cref{sec:exp_add}.

\subsection{Synthetic Experiments}
We study three regimes with categorical instruments (the corresponding experiments with continuous instruments are reported in \Cref{sec:exp_add}):
finite-dimensional instruments with sparse $\beta^*$ (Setting 1),
high-dimensional instruments with dense $\beta^*$ (Setting 2), and
high-dimensional instruments with sparse $\beta^*$ (Setting 3).
Throughout, we use balanced categories and balanced sample sizes ($\tilde n = n$). We also report results as a function of the sample-to-instrument ratio $\frac{n}{m}$
which matches the high-dimensional scaling used in \Cref{sec:hd}. Furthermore, TS-IV and TS-2SLS are numerically equivalent for the categorical setting with balanced samples size. We therefore report TS-2SLS only in the continuous settings. For all experiments, shaded regions indicate confidence intervals around the mean computed over $50$ independent runs.

\begin{figure}[t]
    \centering
    \includegraphics[width=0.9\linewidth]{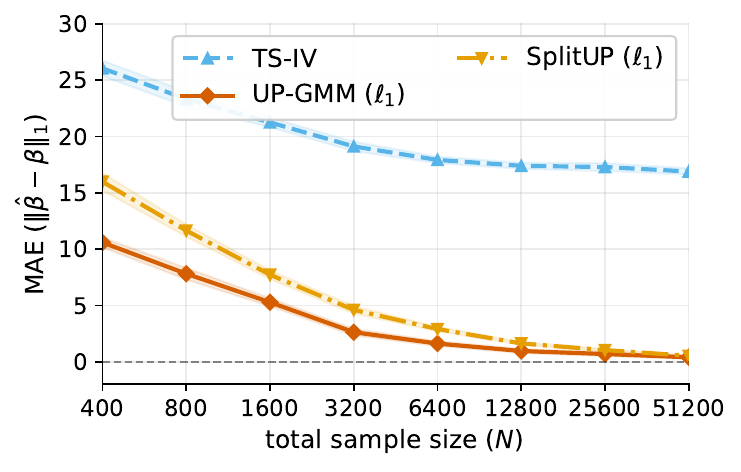}
    \caption{
    \textbf{Setting 1 (finite-dimensional instruments; sparse $\beta^*$).} We compare TS-IV, \estnfd{} (with $\ell_1$ regularization), and \estn{} (with $\ell_1$ regularization) on data with finite-dimensional instruments and a sparse $\beta^*$. Both \estnfd{} and \estn{} are consistent, whereas the estimation error of TS-IV does not vanish even at large sample sizes. Because this setting is low-dimensional, the bias correction in \estn{} is unnecessary and because of the increased variance \estn{} performs worse than \estnfd{}, especially for small sample sizes.
    }
    \label{fig:exp01}
\end{figure}

In particular, for \estn{} we consider (a) Monte Carlo averaging over random splits (with $H$ splits) and (b) the closed-form infinite-split version, denoted \estn{} (analytic). As shown in \Cref{fig:agreement}, \estn{} with $H=10$ and \estn{} (analytic) are numerically indistinguishable in our setting; we therefore report \estn{} (analytic) in all remaining plots.

\paragraph{Setting 1: Finite-dimensional instruments, sparse $\beta^*$.}
Results are shown in \Cref{fig:exp01}. In this regime, the moment matrix can be rank-deficient even asymptotically, so the unregularized TS-IV estimator does not leverage sparsity and fails to recover $\beta^*$ reliably. In contrast, both \estnfd{} and \estn{} exploit the sparse structure and empirically approach the oracle target as $n$ increases. For small and moderate sample sizes, \estnfd{} typically performs slightly better than \estn{}: cross-fitting is not needed to remove high-dimensional measurement-error bias here and the splitting in \estn{} slightly increases variance.

\begin{figure}
        \centering
            \includegraphics[width=0.99\linewidth]{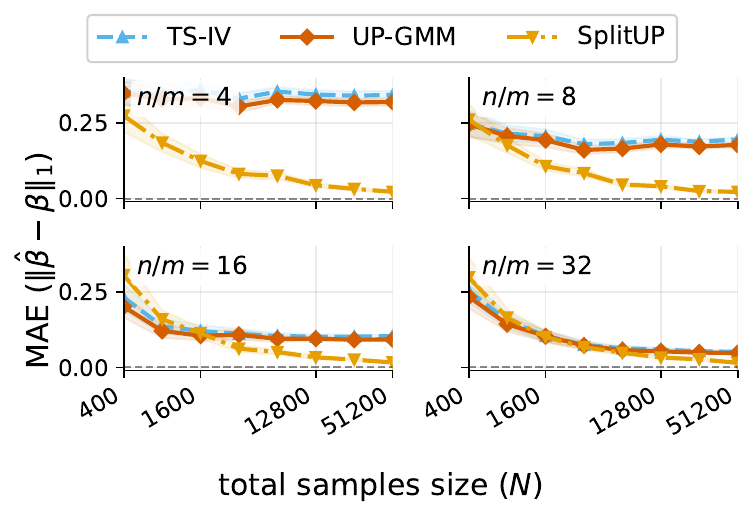}
        \caption{\textbf{Setting 2 (high-dimensional instruments; dense $\beta^*$).} In the high-dimensional instrument regime, the naive plug-in denominator induces a persistent measurement-error bias in two-sample IV, so both TS-IV and \estnfd{} remain asymptotically biased. The cross-moment denominator in \estn{} removes this bias and is the only method that is consistent in this setting. Consistent with the theory, the bias of the naive estimators decreases as $n/m$ increases, i.e., as the problem becomes less high-dimensional.}
        \label{fig:exp02}
\end{figure}

\paragraph{Setting 2: High-dimensional instruments, dense $\beta^*$.}
Results are shown in \Cref{fig:exp02}. In the high-dimensional instrument regime, the naive plug-in denominator induces a persistent measurement-error bias in two-sample IV (see \Cref{sec:42}), making both TS-IV and \estnfd{} asymptotically biased. The cross-moment construction in \estn{} removes this bias and is the only method that remains consistent in this setting. As predicted by the theory, the bias of the naive estimators decreases as $n/m$ increases, i.e., as the problem becomes less high-dimensional.

\begin{figure}[ht]
    \centering
        \includegraphics[width=0.99\linewidth]{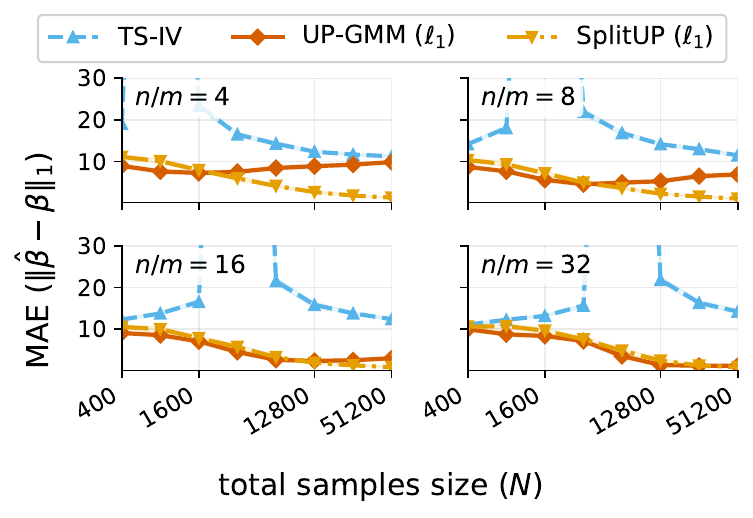}
    \caption{\textbf{Setting 3 (high-dimensional instruments; sparse $\beta^*$).} This setting combines sparse identification with high-dimensional measurement-error bias from plug-in denominators. As a result, TS-IV suffers from both effects, while \estnfd{} addresses the sparsity aspect but still inherits the high-dimensional bias. By combining sparsity with a cross-fit denominator, \estn{} is consistent and achieves the smallest error as $n$ grows. We also observe a transient peaking phenomenon for TS-IV, where the MAE sharply increases at intermediate sample sizes before decreasing again; this is driven by near-singularity of the dense plug-in Gram matrix in the low-rank first-stage regime, and the peak location shifts with $n/m$ (see \Cref{app:tsiv-peaks}).}
    \label{fig:exp05}
\end{figure}

\paragraph{Setting 3: High-dimensional instruments, sparse $\beta^*$.}
Results are shown in \Cref{fig:exp05}. This setting combines the challenges from Settings 1 and 2: identification relies on sparsity, and high-dimensional instruments induce measurement-error bias for plug-in denominators. Consequently, TS-IV suffers from both effects, while \estnfd{} addresses the sparse identification aspect but still inherits the high-dimensional measurement-error bias. By combining sparsity with a cross-fit denominator, \estn{} is consistent and empirically achieves the smallest error as $n$ grows. We also observe a transient `peaking' of TS-IV in Setting~3, where the MAE sharply increases at intermediate sample sizes before decreasing again. This effect is caused by near-singularity of the dense plug-in Gram matrix in the low-rank first-stage regime and its location shifts with $n/m$ because the instability is governed by a critical instrument dimension $m_\star$ and thus occurs around $n_\star\approx (n/m)\,m_\star$; see \Cref{app:tsiv-peaks}.

\bibliographystyle{plainnat}
\bibliography{references}

\clearpage

\appendix
\onecolumn

\section[ff]{Examples of Distributions Satisfying \Cref{ass:1}}
\label{sec:more_graphs}

\Cref{ass:1} (and \Cref{eq:scm}) is rather general: many distributions satisfy \Cref{ass:1} while deviating substantially from the standard two-sample IV setup. In this section, we assume that the data are generated by structural causal models (SCMs) whose induced graphs are shown in \Cref{fig:dag_1}. The parameter of interest is $\beta^*$. We assume that the instrument $I$ affects $X$ and $\tilde X$ in the same way up to second moments, which we depict by a shared coefficient $\gamma$.

\paragraph{Example 1.}
In this example there is no $\tilde Y$ variable. The variables $X$, $Y$, and $\tilde X$ are confounded by an unobserved variable, and we may have $I=\tilde I$. If the distribution is generated by an SCM with induced graph given by \Cref{fig:dag_1} (left), then \Cref{ass:1} is satisfied. Under the usual additional conditions (a rank condition in the dense setting and a restricted nullspace condition in the sparse setting), $\beta^*$ remains identifiable.

\paragraph{Example 2.}
In addition, there may be an observed confounder $Z$ that affects $X$, $\tilde X$, and $Y$; see \Cref{fig:dag_1} (middle). In this case, \Cref{ass:1} holds conditional on $Z$, i.e.,
$\cov(I, Y \mid \tilde X, Z) = \beta^* \cov(I, \tilde X \mid Z)$.
With the same additional conditions as above (rank in the dense case; restricted nullspace in the sparse case), identification again follows.

\paragraph{Example 3.}
Finally, consider the setting without $\tilde I$ and without $\tilde Y$, where $\tilde X$ directly affects $Y$; see \Cref{fig:dag_1} (right). In general, this violates \Cref{ass:1}. Nevertheless, the adjusted moment relation
$\cov(I, Y \mid \tilde X) = \beta^* \cov(I, \tilde X)$
still holds, so the same identification arguments apply. In particular, under appropriate conditions (a rank condition for $\cov(I, \tilde X)$ in the dense setting, and a restricted nullspace condition for $m \cov(I, \tilde X)^\top \cov(I, \tilde X)$ in the sparse setting), $\beta^*$ is identifiable.

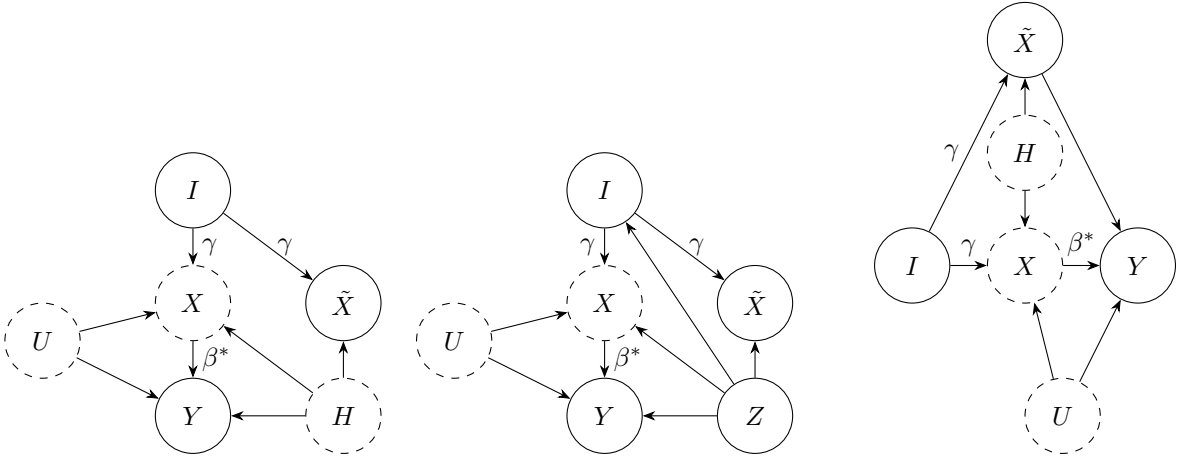
\begin{figure*}[t]
    \begin{subfigure}[t]{0.33\linewidth}
    \centering
        \begin{tikzpicture}[
              >=Stealth,
              node distance=1.2cm and 1.6cm,
              var/.style={draw, circle, inner sep=2pt, minimum width=10mm, minimum height=6mm},
              latent/.style={draw, circle, dashed, inner sep=2pt, minimum width=10mm, minimum height=6mm},
              err/.style={draw, circle, dashed, inner sep=1pt, minimum size=6mm}
            ]
            \node[var] (I) at (-1,2) {$I$};
            
            \node[latent] (H)   at (-3, 0) {$U$};
            \node[latent] (Hr)   at (1, -1) {$H$};
            \node[latent]  (X)   at (-1, 0.5) {$X$};
            \node[var]    (Y)   at (-1,-1) {$Y$};
            
            \node[var]    (Xt)  at ( 1, 0.5) {$\ti{X}$};
            
            \draw[->] (I) --node[midway, right] {$\gamma$} (X);
            \draw[->] (I) --node[midway, right] {$\gamma$}(Xt);
            
            \draw[->] (H)  -- (X);
            \draw[->] (X)  --node[midway, right] {$\beta^*$} (Y);
            \draw[->] (H)  -- (Y);
    
            \draw[->] (Hr)  -- (Y);
            \draw[->] (Hr)  -- (X);
            \draw[->] (Hr)  -- (Xt);
        
        \end{tikzpicture}
    \end{subfigure}\hfill
    \begin{subfigure}[t]{0.33\linewidth}
    \centering
        \begin{tikzpicture}[
              >=Stealth,
              node distance=1.2cm and 1.6cm,
              var/.style={draw, circle, inner sep=2pt, minimum width=10mm, minimum height=6mm},
              latent/.style={draw, circle, dashed, inner sep=2pt, minimum width=10mm, minimum height=6mm},
              err/.style={draw, circle, dashed, inner sep=1pt, minimum size=6mm}
            ]
            \node[var] (I) at (-1,2) {$I$};
            
            \node[latent] (H)   at (-3, 0) {$U$};
            \node[var] (Hr)   at (1, -1) {$Z$};
            \node[latent]  (X)   at (-1, 0.5) {$X$};
            \node[var]    (Y)   at (-1,-1) {$Y$};
            
            \node[var]    (Xt)  at ( 1, 0.5) {$\ti{X}$};
            
            \draw[->] (I) --node[midway, left] {$\gamma$} (X);
            \draw[->] (I) --node[midway, right] {$\gamma$} (Xt);
            
            \draw[->] (H)  -- (X);
            \draw[->] (X)  --node[midway, right] {$\beta^*$} (Y);
            \draw[->] (H)  -- (Y);
    
            \draw[->] (Hr)  -- (Y);
            \draw[->] (Hr)  -- (X);
            \draw[->] (Hr)  -- (Xt);
            \draw[->] (Hr)  -- (I);
        
        \end{tikzpicture}
    \end{subfigure}
    \hfill
    \begin{subfigure}[t]{0.33\linewidth}
    \centering
        \begin{tikzpicture}[
              >=Stealth,
              node distance=1.2cm and 1.6cm,
              var/.style={draw, circle, inner sep=2pt, minimum width=10mm, minimum height=6mm},
              latent/.style={draw, circle, dashed, inner sep=2pt, minimum width=10mm, minimum height=6mm},
              err/.style={draw, circle, dashed, inner sep=1pt, minimum size=6mm}
            ]
            \node[var] (I) at (-2,-1) {$I$};
            
            \node[latent] (H)   at (0, -3) {$U$};
            \node[latent] (Hr)   at (-0.5, .5) {$H$};
            \node[latent]  (X)   at (-0.5, -1) {$X$};
            \node[var]    (Y)   at (1,-1) {$Y$};
            
            \node[var]    (Xt)  at (-0.5, 2) {$\ti{X}$};
            
            \draw[->] (I) --node[midway, above] {$\gamma$} (X);
            \draw[->] (I) --node[midway, left] {$\gamma$} (Xt);
            
            \draw[->] (H)  -- (X);
            \draw[->] (X)  --node[midway, above] {$\beta^*$} (Y);
            \draw[->] (H)  -- (Y);
    
            \draw[->] (Xt)  -- (Y);
            \draw[->] (Hr)  -- (X);
            \draw[->] (Hr)  -- (Xt);
        
        \end{tikzpicture}
    \end{subfigure}
    \caption{Various causal graphs induced by SCMs. Although the shown graphs substantially deviate from the standard two-sample IV setting identification (under appropriate rank conditions) is still possible within our setup.}
    \label{fig:dag_1}
\end{figure*}

\section[g]{Categorical Instruments: Environments without Features as Instruments} \label{sec:categor} 
Assume $I$ is categorical with $m \in \mathbb{N}$ domains and is encoded one–hot as $I\in\{e_1,\dots,e_m\}\subset\mathbb R^m$, where $e_k$ is the $k$th standard basis vector. Suppose that the domain frequencies match across the two samples, with all domains equally likely:
$
    \mathbb{P}(I=e_k)=\mathbb{P}(\ti I=e_k)=\frac{1}{m}\quad\text{for all }k\in[m].
$
Let $n_k=\sum_{i=1}^n \mathbbm 1\{I_i=e_k\}$ and $\ti n_k=\sum_{j=1}^{\ti n}\mathbbm 1\{\ti I_j=e_k\}$ denote the domain counts in the $Y$– and $X$–samples, respectively, and define the within–domain sample means
\begin{equation*}
    \bar Y_k=\frac{1}{n_k}\sum_{i:I_i=e_k} Y_i,
    \qquad
    \bar X_k=\frac{1}{\ti n_k}\sum_{j:\ti I_j=e_k} \ti X_j.
\end{equation*}
under equal domain probabilities and matched empirical shares, $n_k/n=\ti n_k/\ti n=1/m$ for all $k$.
Then \textsc{TS-2SLS}, \textsc{TS-IV} and \estnfd{} ($\gmm(\mathrm{Id})$) are the same estimator and are equivalent to OLS on the category means:
\begin{equation*}
    \gmm(\mathrm{Id})
    =\Big(\sum_{k=1}^m \bar X_k \bar X_k^{\top}\Big)^{-1}\Big(\sum_{k=1}^m \bar X_k \bar Y_k\Big).
\end{equation*}

\section{Algorithm Details}
\label{sec:alg_details}

In total we consider 6 algorithms: TS-IV, TS-2SLS, \estnfd{}, \estn{}, \estn{} (analytic) and naive OLS. For each method, we give a short intuitive description and a compact pseudocode block.

\paragraph{TS-IV.}
TS-IV solves the unpaired moment equation $\cov(I,Y)=\cov(\tilde I,\tilde X)\beta$ by plugging in empirical cross-covariances from the two samples. For numerical stability we use a small ($10^{-10}$) ridge-stabilization. Note that \estnfd{} with $W_N = \mathrm{Id}$ is equivalent to TS-IV. The pseudocode for TS-IV is given in \Cref{alg:tsiv}.

\begin{algorithm}[t]
\caption{TS-IV (ridge-stabilized)}
\label{alg:tsiv}
\begin{algorithmic}[1]
\STATE \textbf{Input:} $\{(I_i,Y_i)\}_{i=1}^n$, $\{(\tilde I_j,\tilde X_j)\}_{j=1}^{\tilde n}$, ridge $\lambda>0$
\STATE Compute $a \coloneqq \widehat{\cov}(I,Y)\in\mathbb R^m$
\STATE Compute $B \coloneqq \widehat{\cov}(\tilde I,\tilde X)\in\mathbb R^{m\times d}$
\STATE Solve $(B^\top B+\lambda \mathrm{Id}_d)\hat\beta = B^\top a$
\STATE \textbf{Output:} $\hat\beta$
\end{algorithmic}
\end{algorithm}

\paragraph{TS-2SLS.}
TS-2SLS is two-stage least squares with unpaired data: first learn the mapping from instruments to covariates on the $(\tilde I,\tilde X)$ sample, then use instruments in the $(I,Y)$ sample to predict the missing covariates, and finally regress $Y$ on the predicted covariates. We again add a small ridge penalty ($10^{-10}$) for numerical stability. The pseudocode for TS-2SLS is given in \Cref{alg:ts2sls}.

\begin{algorithm}[t]
\caption{TS-2SLS (ridge-stabilized)}
\label{alg:ts2sls}
\begin{algorithmic}[1]
\STATE \textbf{Input:} $\{(I_i,Y_i)\}_{i=1}^n$, $\{(\tilde I_j,\tilde X_j)\}_{j=1}^{\tilde n}$, ridge $\lambda>0$
\STATE First stage: $\Gamma \coloneqq (\tilde I^\top \tilde I+\lambda I_m)^{-1}\tilde I^\top \tilde X$
\STATE Predict in $Y$-sample: $\hat X \coloneqq I\Gamma$
\STATE Second stage: solve $(\hat X^\top \hat X+\lambda \mathrm{Id}_d)\hat\beta = \hat X^\top Y$
\STATE \textbf{Output:} $\hat\beta$
\end{algorithmic}
\end{algorithm}

\paragraph{\estnfd{}.}
\estnfd{} is a GMM estimator for the same unpaired moment condition as TS-IV, but it optionally uses an estimated optimal weight matrix and includes $\ell_1$-regularization for settings with sparse $\beta^*$. Intuitively, some moment coordinates are noisier than others because they come from two different samples; optimal weighting downweights noisy moments. A sparse variant adds an $\ell_1$ penalty and can refit on the selected support. We add a small ridge penalty ($10^{-10}$) for numerical stability. For all experiments we set set \texttt{optimal weight = True} and if $\beta^*$ is sparse we additionally set \texttt{l1 = True} and \texttt{post-refit = True}. The pseudocode for \estnfd{} is given in \Cref{alg:upgmm}.

\begin{algorithm}[t]
\caption{\estnfd{}}
\label{alg:upgmm}
\begin{algorithmic}[1]
\STATE \textbf{Input:} data, ridge $\lambda>0$, option \texttt{optimal weight}, option \texttt{l1}, option \texttt{post-refit}
\STATE $a \coloneqq \widehat{\cov}(I,Y)$, $B \coloneqq \widehat{\cov}(\tilde I,\tilde X)$
\STATE Init: $\hat\beta^{(0)} \coloneqq (B^\top B+\lambda \mathrm{Id}_d)^{-1}B^\top a$
\IF{\texttt{optimal weight}}
\STATE Estimate moment covariance $\hat\Omega$ using $IY$ and $\tilde I(\tilde X^\top \hat\beta^{(0)})$ \hfill \textit{(see \Cref{eq:OmegaHat-explicit})}
\STATE Set $W \coloneqq (\hat\Omega+\lambda I_m)^{-1}$
\ELSE
\STATE Set $W \coloneqq I_m$
\ENDIF
\IF{\texttt{l1} is false}
\STATE Solve $(B^\top W B+\lambda \mathrm{Id}_d)\hat\beta = B^\top W a$ 
\ELSE
\STATE Solve $\hat\beta \in \arg\min_\beta \frac12\|W^{1/2}(a-B\beta)\|_2^2 + \lambda\|\beta\|_1$
\IF{\texttt{post-refit}}
\STATE Let $\hat S \coloneqq \{j:|\hat\beta_j|>0\}$ and refit dense UP-GMM restricted to $\hat S$
\ENDIF
\ENDIF
\STATE \textbf{Output:} $\hat\beta$
\end{algorithmic}
\end{algorithm}

\paragraph{\estn{}.}
In the high-dimensional instrument regime, the plug-in denominator $B^\top B$ has measurement-error bias that does not vanish with $m$. \estn{} removes this by forming a cross-moment denominator from independent folds of the $(\tilde I,\tilde X)$ sample. It then solves $C_{XX}\beta=C_{XY}$, optionally with $\ell_1$ regularization and post-refit. We add a small ridge penalty ($10^{-10}$) for numerical stability. For all experiments we choose $K = 2$, $H = 10$, and for sparse $\beta^*$ we additionally set \texttt{l1 = True} and \texttt{post-refit = True}.  The pseudocode for \estn{} is given in \Cref{alg:upgmmhd}.

\begin{algorithm}[t]
\caption{\textsc{FoldCrossCov}: compute $B_k=\widehat{\cov}_k(\tilde I,\tilde X)$}
\label{alg:fold-crosscov}
\begin{algorithmic}[1]
\STATE \textbf{Input:} fold data $\{(\tilde I_j,\tilde X_j)\}_{j\in F_k}$ with $|F_k|=n_k$
\STATE \textbf{Output:} $B_k\in\mathbb R^{m\times d}$
\IF{$\tilde I$ is one-hot (categorical)}
    \STATE Compute fold mean $\bar X \coloneqq \frac{1}{n_k}\sum_{j\in F_k}\tilde X_j$
    \FOR{each environment $e\in[m]$}
        \STATE $J_{k,e}\coloneqq\{j\in F_k:\tilde I_j=e_e\}$, \quad $n_{k,e}\coloneqq |J_{k,e}|$, \quad $p_{k,e}\coloneqq n_{k,e}/n_k$
        \IF{$n_{k,e}>0$}
            \STATE $\bar X_{k,e}\coloneqq \frac{1}{n_{k,e}}\sum_{j\in J_{k,e}}\tilde X_j$
        \ELSE
            \STATE $\bar X_{k,e}\coloneqq \bar X$
        \ENDIF
        \STATE Set row $e$ of $B_k$ as $(B_k)_{e,:}\coloneqq p_{k,e}\,(\bar X_{k,e}-\bar X)^\top$
    \ENDFOR
\ELSE
    \STATE Center columns: $\tilde I_c\coloneqq \tilde I-\frac{1}{n_k}\sum_{j\in F_k}\tilde I_j$, \quad $\tilde X_c\coloneqq \tilde X-\frac{1}{n_k}\sum_{j\in F_k}\tilde X_j$
    \STATE $B_k \coloneqq \frac{1}{n_k}\tilde I_c^\top \tilde X_c$
\ENDIF
\STATE \textbf{return} $B_k$
\end{algorithmic}
\end{algorithm}

\begin{algorithm}[t]
\caption{\estn{} }
\label{alg:upgmmhd}
\begin{algorithmic}[1]
\STATE \textbf{Input:} data, folds $K\ge 2$, redraws $H$, ridge $\lambda>0$, option \texttt{l1}, option \texttt{post-refit}, \texttt{optimal weight}
\STATE $a \coloneqq \widehat{\cov}(I,Y)$, $B \coloneqq \widehat{\cov}(\tilde I,\tilde X)$
\STATE $C_{XY} \coloneqq m\,B^\top a$
\STATE Initialize $C_{XX}\coloneqq 0$
\FOR{$b=1$ to $H$}
\STATE Split $\{1,\dots,\tilde n\}$ into $K$ folds (stratify within environments if $I$ is one-hot)
\STATE For each fold $k$, compute $B_k \coloneqq \textsc{FoldCrossCov}\big(\{(\tilde I_j,\tilde X_j)\}_{j\in F_k}\big)$ (Alg.~\ref{alg:fold-crosscov})

\STATE $C_{XX} \coloneqq C_{XX} + \frac{m}{K(K-1)}\sum_{h\neq k} B_h^\top B_k$
\ENDFOR

\IF{\texttt{optimal weight}}
\STATE Estimate moment covariance $\hat\Omega$
\STATE Set $W \coloneqq (\hat\Omega+\lambda I_m)^{-1}$
\ELSE
\STATE Set $W \coloneqq I_m$
\ENDIF

\STATE $C_{XX} \coloneqq \frac{1}{H}C_{XX}$
\IF{\texttt{l1} is false}
\STATE Solve $(C_{XX}^\top W C_{XX}+\lambda \mathrm{Id}_d)\hat\beta = C_{XX}^\top W C_{XY}$
\ELSE
\STATE Solve $\hat\beta \in \arg\min_\beta \frac12\|W^{1/2}(C_{XY}-C_{XX}\beta)\|_2^2 + \lambda\|\beta\|_1$
\IF{\texttt{post-refit}}
\STATE Let $\hat S \coloneqq \{j:|\hat\beta_j|>0\}$ and refit dense \estn{} on the subset $\hat{S}$
\ENDIF
\ENDIF
\STATE \textbf{Output:} $\hat\beta$
\end{algorithmic}
\end{algorithm}

\paragraph{\estn{} (analytic).}
\estn{} (analytic) replaces Monte Carlo splitting by the closed-form infinite-split limit of the cross-fit denominator. Intuitively, it equals the usual quadratic form minus a self-inner-product correction that removes the leading measurement-error bias. We add a small ridge penalty ($10^{-10}$) for numerical stability.
For sparse $\beta^*$ we set \texttt{l1 = True} and \texttt{post-refit = True}. The pseudocode for \estn{} is given in \Cref{alg:upgmmhdana}.

\begin{algorithm}[t]
\caption{\textsc{CrossFoldDenom}: compute $C_{XX}$ via cross-fold cross-moments}
\label{alg:crossfold-denom}
\begin{algorithmic}[1]
\STATE \textbf{Input:} $\{(\tilde I_j,\tilde X_j)\}_{j=1}^{\tilde n}$, folds $K\ge 2$, redraws $B\ge 1$
\STATE \textbf{Output:} $C_{XX}\in\mathbb R^{d\times d}$
\STATE Initialize accumulator $A \coloneqq 0_{d\times d}$
\FOR{$b=1$ to $B$}
    \IF{$\tilde I$ is one-hot (categorical)}
        \STATE For each environment $e\in[m]$, collect indices $J_e\coloneqq\{j:\tilde I_j=e_e\}$ and randomly split $J_e$ into $K$ parts $J_{e,1},\dots,J_{e,K}$
        \STATE Define folds $F_k \coloneqq \bigcup_{e=1}^m J_{e,k}$ for $k\in[K]$ \COMMENT{stratified within environments}
    \ELSE
        \STATE Randomly permute $\{1,\dots,\tilde n\}$ and split into $K$ folds $F_1,\dots,F_K$
    \ENDIF
    \STATE For each fold $k$, compute $B_k \coloneqq \textsc{FoldCrossCov}\big(\{(\tilde I_j,\tilde X_j)\}_{j\in F_k}\big)$ (Alg.~\ref{alg:fold-crosscov})
    \STATE $S \coloneqq 0_{d\times d}$
    \FOR{$k=1$ to $K$}
        \FOR{$h=1$ to $K$}
            \IF{$h\neq k$}
                \STATE $S \coloneqq S + B_h^\top B_k$
            \ENDIF
        \ENDFOR
    \ENDFOR
    \STATE $A \coloneqq A + \frac{m}{K(K-1)}\,S$
\ENDFOR
\STATE $C_{XX} \coloneqq \frac{1}{B}A$
\STATE \textbf{return} $C_{XX}$
\end{algorithmic}
\end{algorithm}

\begin{algorithm}[t]
\caption{\estn{} (analytic)}
\label{alg:upgmmhdana}
\begin{algorithmic}[1]
\STATE \textbf{Input:} data, ridge $\lambda>0$, option \texttt{l1}, option \texttt{post-refit}, \texttt{optimal weight}
\STATE Compute $C_{XY} \coloneqq m\,\widehat{\cov}(\tilde I,\tilde X)^\top \widehat{\cov}(I,Y)$
\STATE $C_{XX} \coloneqq \textsc{CrossFoldDenom}\big(\{(\tilde I_j,\tilde X_j)\}_{j=1}^{\tilde n}, K, B\big)$ (\Cref{alg:crossfold-denom})
\STATE Set $C_{XX} \coloneqq C_{XX}$

\IF{\texttt{optimal weight}}
\STATE Estimate moment covariance $\hat\Omega$
\STATE Set $W \coloneqq (\hat\Omega+\lambda I_m)^{-1}$
\ELSE
\STATE Set $W \coloneqq I_m$
\ENDIF

\IF{\texttt{l1} is false}
\STATE Solve $(C_{XX}^\top W C_{XX}+\lambda \mathrm{Id}_d)\hat\beta = C_{XX}^\top W C_{XY}$
\ELSE
\STATE Solve $\hat\beta \in \arg\min_\beta \frac12\|W^{1/2}(C_{XY}-C_{XX}\beta)\|_2^2 + \lambda\|\beta\|_1$
\IF{\texttt{post-refit}}
\STATE Let $\hat S \coloneqq \{j:|\hat\beta_j|>0\}$ and refit dense \estn{} (analytic) on the subset $\hat{S}$
\ENDIF
\ENDIF
\STATE \textbf{Output:} $\hat\beta$
\end{algorithmic}
\end{algorithm}

\paragraph{Naive OLS.}
Naive OLS randomly pairs $\tilde X$ rows with $Y$ rows and runs OLS. This ignores the missing joint structure and does not target the IV moment condition; it serves only as a baseline.
\begin{algorithm}[t]
\caption{Naive OLS via random pairing}
\begin{algorithmic}[1]
\STATE \textbf{Input:} $\{Y_i\}_{i=1}^n$, $\{\tilde X_j\}_{j=1}^{\tilde n}$
\STATE Set $n' \coloneqq \min(n,\tilde n)$ and sample indices $\mathcal I_X,\mathcal I_Y$ of size $n'$
\STATE Form paired matrices $X_p \coloneqq \tilde X_{\mathcal I_X}$ and $Y_p \coloneqq Y_{\mathcal I_Y}$
\STATE Center columns of $X_p$ and center $Y_p$
\STATE Output $\hat\beta \coloneqq (X_p^\top X_p)^\dagger X_p^\top Y_p$
\end{algorithmic}
\end{algorithm}

\section{Data Generating Processes}
\label{sec:dg_details}

All synthetic experiments follow the unpaired IV model in \Cref{eq:scm}. We generate two independent samples: a $Y$-sample $\{(I_i,Y_i)\}_{i=1}^n$ and an $X$-sample $\{(\tilde I_j,\tilde X_j)\}_{j=1}^{\tilde n}$, where $X$ is latent in the $Y$-sample and $\tilde Y$ is latent in the $X$-sample. Hidden confounding is introduced via a latent variable $U$ that affects both $X$ and $Y$. Throughout, we use balanced sample sizes of the form
\begin{equation*}
    n = m r,
    \qquad
    \tilde n = m \ti r,
\end{equation*}
where $m$ denotes the number (or dimension) of instruments and $r,\ti r$ control the sample-to-instrument ratios. In all settings, the two samples share the same first-stage cross-moment (Assumption~\ref{ass:1}.\ref{ass:1:1}) by construction.

\paragraph{Common structural equations.}
We generate latent covariates and outcomes via
\begin{equation*}
    X = \mu(I) + \gamma_x U + \varepsilon_x,
    \qquad
    Y = X^\top \beta^* + \gamma_y U + \varepsilon_y,
    \qquad
    \ti X = \mu( \ti I) + \gamma_x \ti U + \ti \varepsilon_x,
\end{equation*}
where $U, \ti U\sim\mathcal N(0,\sigma_u^2)$ is an unobserved confounder and $\ti \varepsilon_x, \varepsilon_x,\varepsilon_y$ are mean-zero noise terms. The dependence of $Y$ on $U$ implies endogeneity ($\mathbb E[\varepsilon \mid X]\neq 0$), while exogeneity with respect to the instrument is enforced by construction ($\mathbb E[\varepsilon\mid I]=0$).

\subsection{Categorical instruments (one-hot environments)}
\label{sec:dg_categorical}

\paragraph{Instrument.}
For categorical instruments, $I\in\{e_1,\dots,e_m\}\subset\mathbb R^m$ is one-hot with uniform environment probability. The $Y$-sample and $X$-sample are balanced across environments: for each environment $e\in[m]$ we generate exactly $r$ observations in the $Y$-sample and $\ti r$ observations in the $X$-sample.

\paragraph{Environment-specific first stage and heteroskedasticity.}
We draw environment means $\mu_e\in\mathbb R^d$ i.i.d.\ as
\begin{equation*}
    \mu_e \sim \mathcal N(0,\mathrm{Id}_d),\qquad e\in[m].
\end{equation*}
To introduce realistic heteroskedasticity, we draw environment-specific noise scales
\begin{equation*}
    \sigma_{x,e} \sim \sigma_x \cdot \mathrm{LogNormal}(0,0.5),
    \qquad
    \sigma_{\varepsilon,e} \sim \sigma_\varepsilon \cdot \mathrm{LogNormal}(0,0.5),
\end{equation*}
clipped to a fixed range and renormalized to keep the average scale constant.

\paragraph{Sampling.}
For each $Y$-sample observation in environment $e$ we sample
\begin{equation*}
    U\sim\mathcal N(0,\sigma_u^2),\quad
    \varepsilon_x\sim\mathcal N(0,\sigma_{x,e}^2 \mathrm{Id}_d),\quad
    \varepsilon_y\sim\mathcal N(0,\sigma_{\varepsilon,e}^2),
\end{equation*}
    set $X_{\mathrm{lat}} = \mu_e + \gamma_x U + \varepsilon_x$, and output
\begin{equation*}
    Y = X_{\mathrm{lat}}^\top \beta^* + \gamma_y U + \varepsilon_y.
\end{equation*}
For each $X$-sample observation in the same environment $e$ we sample $\ti U$ and $\ti \varepsilon_x$ analogously and output
\begin{equation*}
    \tilde X = \mu_e + \gamma_x \ti U + \ti \varepsilon_x.
\end{equation*}
This construction ensures $\cov(I,X)=\cov(\tilde I,\tilde X)$.

\paragraph{Setting 1 (categorical): finite-dimensional instruments, sparse $\beta^*$.}
We fix the number of environments $m$ and dimension $d$, and choose $\beta^*$ to be sparse with $s^*$ nonzeros:
\begin{equation*}
\|\beta^*\|_0 = s^*,\qquad
(\beta^*)_j \in [-1, -0.5] \cup [0.5, 1] \text{ uniformly}.
\end{equation*}
We increase the sample size through $r,\ti r$ (equivalently $n,\tilde n\to\infty$ with fixed $m$), matching the finite-dimensional instrument regime.

\paragraph{Setting 2 (categorical): high-dimensional instruments, dense $\beta^*$.}
We consider the high-dimensional instrument regime by increasing $m$ while keeping the ratios $r=n/m$ and $\ti r=\tilde n/m$ fixed. We consider $d$. The causal effect $\beta^*$ is dense $((\beta^*)_j \in [-1, -0.5] \cup [0.5, 1]$ uniformly), and performance is reported as a function of $n/m$. This is the regime where the plug-in denominator in unpaired IV exhibits persistent measurement-error bias.

\paragraph{Setting 3 (categorical): high-dimensional instruments, sparse $\beta^*$.}
This setting combines high-dimensional instruments ($m\to\infty$ with fixed $n/m$) and sparse causal effects ($\|\beta^*\|_0=s^*$). To induce weak and low-rank first stages, we replace the i.i.d.\ environment means by a low-rank construction:
\begin{equation*}
    Z_e\sim\mathcal N(0,\mathrm{Id}_k),
    \qquad
    \mu_e = Z_e A^\top,\qquad
    A\in\mathbb R^{d\times k}\ \text{fixed},\ k\ll d.
\end{equation*}
Equivalently, $\{\mu_e\}_{e=1}^m$ lie in the $k$-dimensional subspace spanned by the columns of $A$, yielding a rank-constrained moment matrix and making sparsity essential for identification. Noise scales and confounding are generated as in Setting~1.

\subsection{Continuous Instruments}
\label{sec:dg_continuous}

\paragraph{Instrument and first-stage map.}
For continuous instruments, we draw (independently)
\begin{equation*}
    I \in \mathbb R^m,\qquad I \sim \mathcal N\Big(0,\frac{1}{m} \mathrm{Id}_m\Big), \qquad \ti I \in \mathbb R^m,\qquad \ti I \sim \mathcal N\Big(0,\frac{1}{m} \mathrm{Id}_m\Big),
\end{equation*}
and generate covariates through a linear first stage with a shared matrix $\Pi\in\mathbb R^{m\times d}$:
\begin{equation*}
    X = I\Pi + \gamma_x U + \varepsilon_x,
    \qquad
    Y = X^\top\beta^* + \gamma_y U + \varepsilon_y,
    \qquad
    \ti X = \ti I \Pi + \gamma_x \ti U + \ti \varepsilon_x.
\end{equation*}
The matrix $\Pi$ is sampled once per dataset and shared across the $Y$-sample and $\ti X$-sample, ensuring $\cov(I,X)=\cov(\tilde I,\tilde X)$ by construction.

\paragraph{Coordinate-dependent heteroskedasticity.}
To mimic heterogeneous noise across instrument coordinates, each sample point is assigned a `dominant coordinate'
\begin{equation*}
c(i) \coloneqq \arg\max_{\ell\in[m]} |I_{i\ell}|,
\end{equation*}
and the noise scales depend on $c(i)$:
\begin{equation*}
\varepsilon_y \sim \mathcal N\big(0,\sigma_{\varepsilon,c(i)}^2\big),
\qquad
(\varepsilon_x)_t \sim \mathcal N\big(0,\sigma_{x,c(i)}^2\big)\ \text{independently for }t\in[d],
\end{equation*}
where $\{\sigma_{\varepsilon,\ell}\}_{\ell=1}^m$ and $\{\sigma_{x,\ell}\}_{\ell=1}^m$ are sampled from bounded lognormal distributions.

\paragraph{Setting 1 (continuous): finite-dimensional instruments, sparse $\beta^*$.}
We fix $m$ and $d$, choose $\beta^*$ sparse with $\|\beta^*\|_0=s^*$, and increase sample size via $r,\ti r$ (thus $n,\tilde n\to\infty$ with fixed $m$). The first-stage matrix $\Pi$ is dense i.i.d.\ Gaussian, which yields a full-rank population first stage when $m$ is sufficiently large, but the regime of interest keeps $d$ larger than $m$ so sparse structure is required.

\paragraph{Setting 2 (continuous): high-dimensional instruments, dense $\beta^*$.}
We increase $m$ while keeping $n/m$ and $\tilde n/m$ fixed and use a dense $\beta^*$. The persistence of measurement-error bias for plug-in denominators in two-sample IV carries over to this continuous-instrument regime, motivating cross-moment denominators.

\paragraph{Setting 3 (continuous): high-dimensional instruments, sparse $\beta^*$.}
We again take $m\to\infty$ with fixed ratios $n/m$ and $\tilde n/m$, but use a sparse $\beta^*$ with $\|\beta^*\|_0=s^*$. To create a low-rank first stage (analogous to the low-rank environment means in the categorical case), we draw a fixed matrix $A\in\mathbb R^{d\times k}$ with $k\ll d$ and set
\begin{equation*}
    Z \in \mathbb R^{m\times k},\quad Z_{\ell}\sim\mathcal N(0,\mathrm{Id}_k),\qquad
    \Pi \coloneqq \pi_{\mathrm{scale}}\,(Z A^\top),
\end{equation*}
so $\Pi$ has rank at most $k$. Confounding and heteroskedasticity are generated as in setting 2.

\subsection{Choices of Constants}
For all experiments, we set $\ti n = n$ and $\ti r = r$ and 
\begin{align*}
    \gamma_x \coloneqq \frac{1}{5} \qquad
    \gamma_y \coloneqq \frac{1}{5} \qquad
    \sigma_u \coloneqq \frac{1}{5} \qquad
    \sigma_x \coloneqq 1 \qquad
    \sigma_{\varepsilon} &\coloneqq \frac{1}{5}
\end{align*}
We use the same parameters for the categorical instrument and the continuous instrument experiments.
\paragraph{Setting 1.} We set $m=100$, $d=200$ and $s^*=10$.

\paragraph{Setting 2.} We set $d=2$. 

\paragraph{Setting 3.} We set $k=60$, $d=100$ and $s^*=10$.

\section{Additional Experiments}
\label{sec:exp_add}

\subsection{Agreement between \estn{} and \estn{} (analytic).}
\Cref{fig:agreement} show that the estimates of \estn{} and \estn{} (analytic) agree closely in Setting 2 (with categorical instruments).

\begin{figure*}[t]
    \centering
    \includegraphics[width=0.45\linewidth]{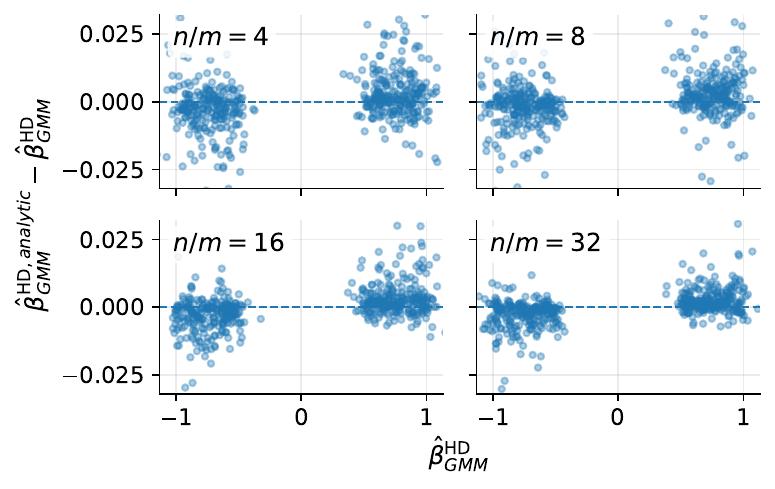}
    \caption{
    \textbf{Agreement between \estn{} and \estn{} (analytic).} In Setting 2, we compare \estn{} to \estn{} (analytic). To mitigate numerical instability when $n/m=32$ and $N$ is small, we remove the $0.1\%$ of points with the largest discrepancy (outliers). The two estimators agree closely, remaining within a $2.5\%$ margin of each other.
    }
    \label{fig:agreement}
\end{figure*}

\subsection{Continuous Instrument}
The experimental details can be found in \Cref{sec:dg_continuous}. The results for Setting 1 are given in \Cref{fig:exp01_cont} and the results of Setting 2 and 3 in \Cref{fig:exp025_cont}.

\begin{figure}
        \centering
            \includegraphics[width=0.49\linewidth]{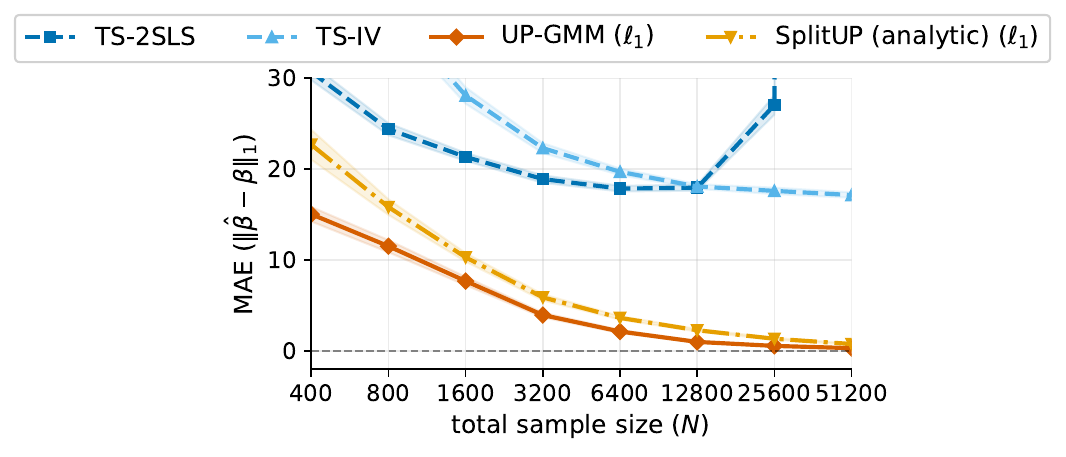}
        \caption{\textbf{Setting 1 (continuous).} \estnfd{} ($\ell_1$-regularized) and \estn{} ($\ell_1$-regularized) are consistent while TS-IV and TS-2SLS are not.}
        \label{fig:exp01_cont}
\end{figure}

\begin{figure*}[t]
    \centering
    \begin{subfigure}[t]{0.47\linewidth}
        \centering
        \includegraphics[width=\linewidth]{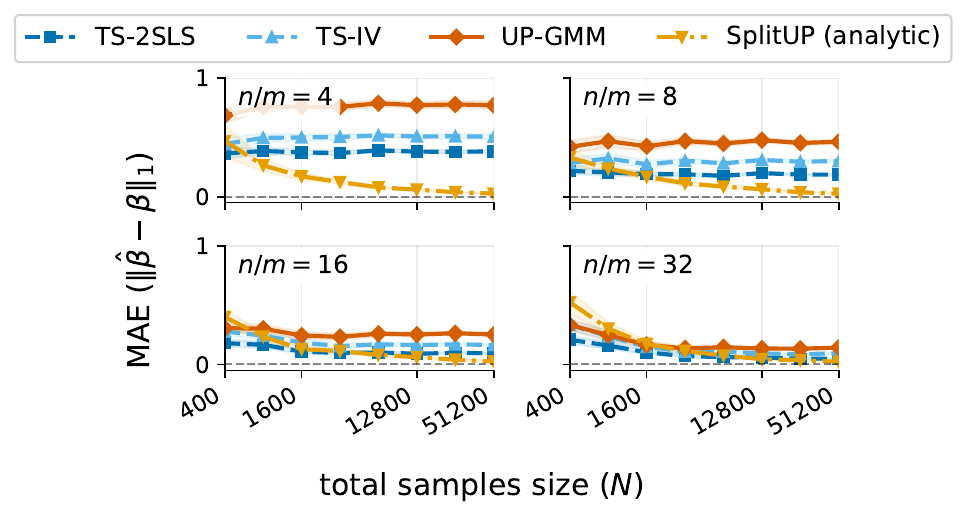}
    \end{subfigure}\hfill
    \begin{subfigure}[t]{0.51\linewidth}
        \centering
        \includegraphics[width=\linewidth]{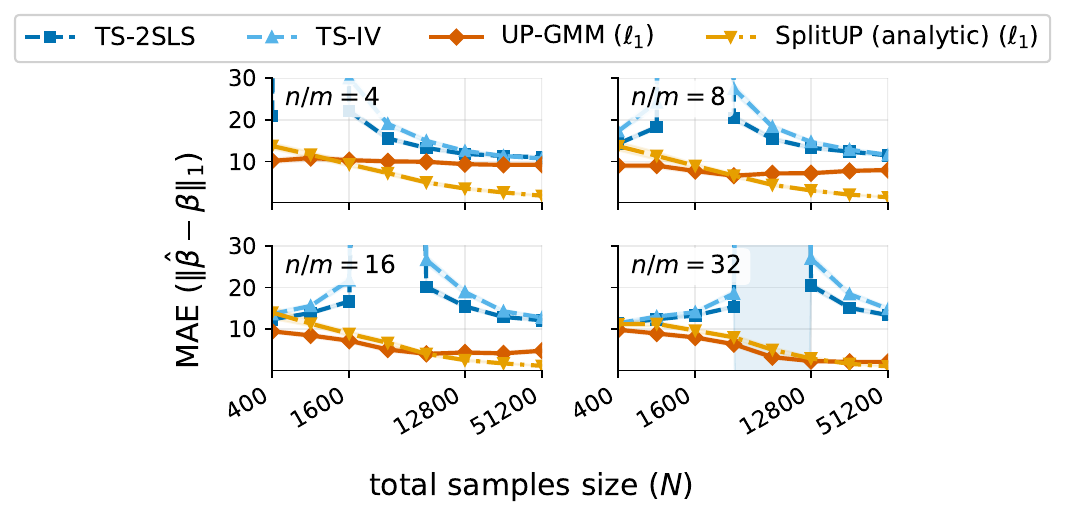}
    \end{subfigure}
    \caption{\textbf{Left: Setting 2 (continuous).} Only \estn{} is consistent, all other estimators remain biased even for large sample-sizes. The bias reduces as $n/m$ increases and the problem becomes less high-dimensional.
    \textbf{Left: Setting 3 (continuous).} Only \estn{} ($\ell_1$-regularized) is consistent, all other estimators remain biased even for large sample-sizes. The bias reduces for \estnfd{} ($\ell_1$-regularized) as $n/m$ increases and the problem becomes less high-dimensional.
    }
    \label{fig:exp025_cont}
\end{figure*}

\section{Explaining the TS-IV Peaking Phenomenon in Setting 3}
\label{app:tsiv-peaks}

In Setting~3 the population first stage is intentionally low-rank: the environment means satisfy
$\mu_e = Z_e A^\top$ with $A\in\mathbb R^{d\times k}$ and $k\ll d$, so $\cov(I,X)$ (and hence $\cov(\tilde I,\tilde X)$) has rank at most $k$.
As a consequence, the population matrix $B_\star\coloneqq\cov(\tilde I,\tilde X)\in\mathbb R^{m\times d}$ has $d-k$ directions with zero signal.
In finite samples, the empirical cross-covariance
$B\coloneqq\widehat{\cov}(\tilde I,\tilde X)=B_\star+E$
contains estimation noise $E$ (coming from within-environment sampling noise and heteroskedasticity), which injects random components into these $d-k$ `null' directions.

TS-IV computes (up to a small ridge) the dense plug-in solution
\begin{equation*}
    \hat\beta_{\mathrm{TS\text{-}IV}}
    =
    (B^\top B+\lambda I_d)^{-1}B^\top a,
    \qquad
    a=\widehat{\cov}(I,Y).    
\end{equation*}

The key issue is conditioning: in Setting~3 the smallest eigenvalues of $B^\top B$ are dominated by the Gram matrix of the noise-only block (the components of $E$ orthogonal to $\mathrm{span}(A)$).
When the number of instruments $m$ is in the critical regime where this noise Gram matrix is barely invertible, $\lambda_{\min}(B^\top B)$ can become very small in a non-negligible fraction of repetitions, so $(B^\top B+\lambda I_d)^{-1}$ amplifies noise and produces occasional extremely large coefficients, leading to a pronounced `peaking' of the MAE.

This peak is specific to Setting~3 because only there the population first stage is rank-deficient ($k<d$), creating a $(d-k)$-dimensional subspace where TS-IV must effectively invert pure noise; in Settings~1--2 (fixed $m$ or very small $d$) this near-singularity mechanism is absent or not swept through.
Finally, the peak `moves to the right' with $n/m$ because in the high-dimensional scaling $n=rm$: the instability is triggered at a roughly fixed critical instrument dimension $m_\star$ (set by the effective null dimension, e.g.\ $m_\star\approx d-k$ up to constants), hence it occurs at sample size $n_\star\approx r\,m_\star$, which increases linearly with $r=n/m$.

\section{Consistency and Asymptotic CIs for the Dense Setting}
\label{sec:dense}

For the optimal choice, estimate the variance of the sample moment at a preliminary consistent estimator $\hat\beta^{(0)}$ (for instance \eqref{eq:betaGMM-explicit} with $W_N=\mathrm{Id}_m$) via
\begin{equation*}
    \widehat{\Omega}_m \coloneqq \frac{1}{n}\sum_{i=1}^n \Big(I_i Y_i - \frac{1}{n}\sum_{r=1}^n I_r Y_r\Big)\Big(I_i Y_i - \frac{1}{n}\sum_{r=1}^n I_r Y_r\Big)^{\top},
\end{equation*}
\begin{equation*}
    \widehat{\Omega}_c\big(\hat\beta^{(0)}\big) \coloneqq \frac{1}{\ti n}\sum_{j=1}^{\ti n} \Big(\ti I_j \ti X_j^{\top}\hat\beta^{(0)} - \Big[\frac{1}{\ti n}\sum_{s=1}^{\ti n} \ti I_s \ti X_s^{\top}\Big]\hat\beta^{(0)}\Big)\Big(\ti I_j \ti X_j^{\top}\hat\beta^{(0)} - \Big[\frac{1}{\ti n}\sum_{s=1}^{\ti n} \ti I_s \ti X_s^{\top}\Big]\hat\beta^{(0)}\Big)^{\top}.
\end{equation*}
Set
\begin{equation}\label{eq:OmegaHat-explicit}
    \widehat{\Omega} \coloneqq \tau_n^{-1}\widehat{\Omega}_m + \ti\tau_n^{-1}\widehat{\Omega}_c\big(\hat\beta^{(0)}\big),\qquad \widehat W \coloneqq \widehat{\Omega}^{-1}.
\end{equation}
The asymptotic variance of the estimator is minimized by replacing $W_N$ with $\widehat W$ in \eqref{eq:betaGMM-explicit}.
Define
\begin{align*}
    V(W_0) &\coloneqq \Big(\cov(\ti I,\ti X)^{\top}W_0 \cov(\ti I,\ti X)\Big)^{-1} \cov(\ti I,\ti X)^{\top}W_0\Omega W_0 \\
    &\quad \cov(\ti I,\ti X)\Big(\cov(\ti I,\ti X)^{\top}W_0 \cov(\ti I,\ti X)\Big)^{-1}.
\end{align*}
\begin{assumption}\label{ass:asympDense}
Assume all of the following:
\begin{enumerate}[label=(\roman*), nosep]
    \item \Cref{ass:1}, the centering convention and with bounded fourth moments.
    \item \label{ass:2:1}$\mathrm{rank}\big(\cov( I, X)\big)=d$.
    \item \label{ass:2:3}$W_N \overset{p}{\to} W_0$ with $W_0$ positive definite.
\end{enumerate}   
\end{assumption}

\begin{proposition}[Consistency and asymptotic normality, dense case; formal version of \Cref{thm:denseAN-informal}]\label{thm:denseAN}
Under \Cref{ass:asympDense}, the estimator \eqref{eq:betaGMM-explicit} is consistent:
\begin{equation}
    \gmm(W_N) \overset{p}{\to} \beta^*.
\end{equation}
Then
\begin{equation}\label{eq:anDense-explicit}
    \sqrt{N}\big(\gmm(W_N)-\beta^*\big) \overset{d}{\to} \mathcal N\Big(0, V(W_0)\Big).
\end{equation}
With the choice $W_0=\Omega^{-1}$ the variance in \eqref{eq:anDense-explicit} is minimized in the usual GMM sense.
Let $\widehat V$ be the empirical version of $V(\Omega^{-1})$. Then
\begin{equation*}
    \sqrt{N} \widehat V^{-1/2}\big(\gmm(W_N)-\beta^*\big) \overset{d}{\to} \mathcal N\Big(0, \mathrm{Id}_d \Big).
\end{equation*}
\end{proposition}

\begin{proof}
    By the strong law, 
    $$\frac{1}{n}\sum_{i=1}^n I_iY_i \overset{p}{\to} \mathbb E[IY] = \cov(I,Y),\qquad \frac{1}{\ti n}\sum_{j=1}^{\ti n}\ti I_j\ti X_j^{\top} \overset{p}{\to} \mathbb E[\ti I \ti X^{\top}] = \cov(\ti I,\ti X).$$
    Therefore $\hat g_N(\beta) \overset{p}{\to} \cov(I,Y) - \cov(\ti I,\ti X)\beta$. The continuous mapping theorem and \Cref{ass:asympDense}~\ref{ass:2:1} imply that the population criterion $Q(\beta) = \big(\cov(I,Y) - \cov(\ti I,\ti X)\beta\big)^{\top}W_0\big(\cov(I,Y) - \cov(\ti I,\ti X)\beta\big)$ has the unique minimizer $\beta^*$. Uniform convergence of the quadratic sample criterion to $Q$ yields consistency of \eqref{eq:betaGMM-explicit}.
    
    For asymptotic normality, write the first–order condition
    \begin{equation*}
        \Big(\frac{1}{\ti n}\textstyle\sum_{j=1}^{\ti n} \ti X_j \ti I_j^{\top}\Big) W_N \hat g_N(\gmm(W_N)) = 0.
    \end{equation*}
    Add and subtract $\beta^*$ inside $\hat g_N$ and rearrange:
    \begin{equation*}
        \Big(\frac{1}{\ti n}\textstyle\sum_{j=1}^{\ti n} \ti X_j \ti I_j^{\top}\Big) W_N \Big(\frac{1}{\ti n}\textstyle\sum_{j=1}^{\ti n} \ti I_j \ti X_j^{\top}\Big)\big(\gmm(W_N)-\beta^*\big) = \Big(\frac{1}{\ti n}\textstyle\sum_{j=1}^{\ti n} \ti X_j \ti I_j^{\top}\Big) W_N \hat g_N(\beta^*).
    \end{equation*}
    By Slutsky, the left matrix converges in probability to $\cov(\ti I,\ti X)^{\top}W_0\cov(\ti I,\ti X)$, which is invertible by \Cref{ass:asympDense}\ref{ass:2:1}. Multiplying both sides by its inverse and by $\sqrt{N}$ yields
    \begin{equation}\label{eq:ANstep}
        \sqrt{N}\big(\gmm(W_N)-\beta^*\big) = \Big(\cov(\ti I,\ti X)^{\top}W_0\cov(\ti I,\ti X)\Big)^{-1}\cov(\ti I,\ti X)^{\top}W_0 \sqrt{N} \hat g_N(\beta^*) + o_{p}(1).
    \end{equation}
    Finally, apply a joint CLT to the two independent samples:
    \begin{equation}
    \label{eq:CLT2}
        \begin{aligned}
        \sqrt{N} \hat g_N(\beta^*) &= \sqrt{N}\Big(\frac{1}{n}\sum_{i=1}^n(I_iY_i-\mathbb E[IY])\Big) - \sqrt{N}\Big(\frac{1}{\ti n}\sum_{j=1}^{\ti n}(\ti I_j\ti X_j^{\top}-\mathbb E[\ti I\ti X^{\top}])\Big)\beta^* \\
        &\overset{d}{\to} \mathcal N\big(0,\tau^{-1}\Omega_m + \ti\tau^{-1}\Omega_c(\beta^*)\big),
        \end{aligned}
    \end{equation}
    because $\sqrt{N/n}\to \tau^{-1/2}$ and $\sqrt{N/\ti n}\to \ti\tau^{-1/2}$. Plug \eqref{eq:CLT2} into \eqref{eq:ANstep} to obtain \eqref{eq:anDense-explicit}.
\end{proof}

\section{Proofs}

\subsection[Proof of Theorem~\ref{thm:iden_dense}]
        {Proof of \Cref{thm:iden_dense}}
\label{sec:p_iden_dense}

    If $\mathrm{rank}(\cov(I, X)) = d$, then $\cov(I, X)$ (and therefore $\cov(\ti I, \ti X)$) has a left inverse and there exists a unique solution $\hat \beta = \beta^*$ to
    \begin{equation*}
        \cov(I, Y) =  \cov(\ti I, \ti X) \beta
    \end{equation*}
    given by
    \begin{equation*}
        \hat \beta = \left( \cov(\ti I, \ti X)^{\top} \cov(\ti I, \ti X) \right)^{-1} \cov(\ti I, \ti X)^{\top}\cov(I, Y) = \beta^*.
    \end{equation*}

    If 
    $$
    \operatorname{rank}\left(\operatorname{cov}(\ti I,\ti X)\right) = \operatorname{rank}\left(\operatorname{cov}( I,X)\right) < d,
    $$
    then 
    $$
    \ker\left(\operatorname{cov}(\ti I,\ti X)\right) \neq \{0\},
    $$
    and for any $h$ in the kernel, $\beta^* + h$ yields the same moment, so $\beta^*$ is not identifiable.

\subsection[Proof of Theorem~\ref{thm:sparse_ident}]{Proof of \Cref{thm:sparse_ident}}
\textbf{(i) $\Rightarrow$ (ii).}
Suppose (ii) fails. Then there exists $0\neq h\in \ker(\cov(I, X))\cap\Sigma_{2s^*}$.
Partition $\operatorname{supp}(h)$ into disjoint $J_1,J_2$ with $|J_i|\leq s^*$ and define
$u,v\in\Sigma_{s^*}$ by $u_{J_1}=h_{J_1}$, $v_{J_2}=-h_{J_2}$, zeros elsewhere.
Then $\cov(I, X)u=\cov(I, X)v$ and $\|u\|_0,\|v\|_0\leq s^*$, but $u\neq v$, contradicting (i) for $\beta^* = v$.
Hence, (ii) must hold.

\textbf{(ii) $\Rightarrow$ (i)}
If $\cov(I, X)\widehat\beta = \cov(I, X)\beta^*$ and $\|\widehat\beta\|_0\leq \|\beta^*\|_0 \leq s^*$, then
$h\coloneqq\widehat\beta-\beta^*\in \ker(\cov(I, X))\cap \Sigma_{2s^*}$; by (ii), $h=0$, so $\widehat\beta=\beta^*$.

\subsection[f]{Proof of \Cref{thm:iden_dense_hd}}

If $\mathrm{rank}(Q) = d$, then there exists a unique solution $\hat \beta = \beta^*$ to
\begin{equation*}
    Q_y =  Q \beta
\end{equation*}
given by
\begin{equation*}
    \hat \beta = Q^{-1} Q_y = \beta^*.
\end{equation*}

If 
$
    \operatorname{rank}\left(Q\right) < d,
$
then 
$
    \ker\left(Q\right) \neq \{0\},
$
and for any $h$ in the kernel, $\beta^* + h$ yields the same moment, so $\beta^*$ is not identifiable.

\subsection[d]{Proof of \Cref{thm:sparse_ident_hd}}
\textbf{(i) $\Rightarrow$ (ii).}
Suppose (ii) fails. Then there exists $0\neq h\in \ker(Q)\cap\Sigma_{2s^*}$.
Partition $\operatorname{supp}(h)$ into disjoint $J_1,J_2$ with $|J_i|\leq s^*$ and define
$u,v\in\Sigma_{s^*}$ by $u_{J_1}=h_{J_1}$, $v_{J_2}=-h_{J_2}$, zeros elsewhere.
Then $Qu=Qv$ and $\|u\|_0,\|v\|_0\leq s^*$, but $u\neq v$, contradicting (i) for $\beta^* = v$.
Hence, (ii) must hold.

\textbf{(ii) $\Rightarrow$ (i)}
If $Q\widehat\beta =Q\beta^*$ and $\|\widehat\beta\|_0\leq \|\beta^*\|_0 \leq s^*$, then
$h\coloneqq\widehat\beta-\beta^*\in \ker(Q)\cap \Sigma_{2s^*}$; by (ii), $h=0$, so $\widehat\beta=\beta^*$.

\subsection[Proof of Theorem~\ref{thm:ratesSparse}]{Proof of \Cref{thm:ratesSparse}}
\label{sec:ratesSparse}

\begin{lemma}
\label{lem:scorebound-finite-m}
    Let $\hat g_N(\beta) \coloneqq \frac{1}{n}\sum_{i=1}^n I_iY_i - \big(\frac{1}{\ti n}\sum_{j=1}^{\ti n}\ti I_j\ti X_j^\top\big)\beta$
    and $\widehat{\cov}(\ti I, \ti X) \coloneqq \frac{1}{\ti n}\sum_{j=1}^{\ti n}\ti I_j\ti X_j^\top$.
    Under Assumption~\ref{ass:RE},
    \begin{equation*}
        \big\|\widehat{\cov}(\ti I, \ti X)^\top W_N \hat g_N(\beta^*)\big\|_\infty = O_p(N^{-1/2}).
    \end{equation*}
\end{lemma}

\begin{proof}
    By the moment condition $\cov(I,Y)=\cov(\ti I,\ti X)\beta^*$, we have 
    \begin{equation*}
        \hat g_N(\beta^*)
        = \widehat{\cov}(I, Y)- \widehat{\cov}(\ti I, \ti X)\beta^*
        = (\widehat{\cov}(I, Y)-\cov(I, Y)) - (\widehat{\cov}(\ti I, \ti X)-\cov(I, X))\beta^*.
    \end{equation*}
    Since $m$ is fixed and fourth moments are bounded, Chebyshev (or a multivariate CLT) yields
    \begin{equation*}
        \|\widehat{\cov}(I, Y)-\cov(I, Y)\|_2 = O_p(n^{-1/2}),
        \qquad
        \|\widehat{\cov}(\ti I, \ti X)-\cov(I, X)\|_{\mathrm{op}} = O_p(\ti n^{-1/2}).
    \end{equation*}
    Using $\|\beta^*\|_2=O(1)$ and $\tau_n,\ti\tau_n\to(0,1)$ gives
    \begin{equation*}
        \|\hat g_N(\beta^*)\|_2
        \le \|\widehat{\cov}(I, Y)-\cov(I, Y)\|_2 + \|\widehat{\cov}(\ti I, \ti X)-\cov(I, X)\|_{\mathrm{op}}\|\beta^*\|_2
        = O_p(n^{-1/2}+\ti n^{-1/2})
        = O_p(N^{-1/2}).
    \end{equation*}
    Moreover, by the law of large numbers in fixed dimensions, $\|\widehat{\cov}(\ti I, \ti X)\|_{\mathrm{op}}=O_p(1)$, and by $W_N\overset{p}{\to}W_0\succ 0$ we have $\|W_N\|_{\mathrm{op}}=O_p(1)$. Therefore,
    \begin{equation*}
        \|\widehat{\cov}(\ti I, \ti X)^\top W_N \hat g_N(\beta^*)\|_\infty
        \le \|\widehat{\cov}(\ti I, \ti X)^\top W_N \hat g_N(\beta^*)\|_2
        \le \|\widehat{\cov}(\ti I, \ti X)\|_{\mathrm{op}}\|W_N\|_{\mathrm{op}}\|\hat g_N(\beta^*)\|_2
        = O_p(N^{-1/2}).
    \end{equation*}
\end{proof}

Let 
$$
    S\coloneqq\operatorname{supp}(\beta^*),\qquad s^* \coloneqq |S|.
$$
The optimality of \eqref{eq:LassoGMM-explicit} yields the basic inequality
\begin{equation}\label{eq:basic-ineq}
    \frac{1}{2}\Big\|W_N^{1/2}\hat g_N(\ivs)\Big\|_2^2 + \lambda_N \|\ivs\|_1 \le \frac{1}{2}\Big\|W_N^{1/2}\hat g_N(\beta^*)\Big\|_2^2 + \lambda_N \|\beta^*\|_1.
\end{equation}
Using the linearization $$\hat g_N(\ivs) = \hat g_N(\beta^*) - \Big(\frac{1}{\ti n}\sum_{j=1}^{\ti n} \ti I_j \ti X_j^{\top}\Big)\big(\ivs-\beta^*\big)$$ and expanding \eqref{eq:basic-ineq}, one obtains
\begin{align}\label{eq:cone-step}
    \big(\ivs-\beta^*\big)^{\top}&\Big(\frac{1}{\ti n}\sum_{j=1}^{\ti n} \ti X_j \ti I_j^{\top}\Big)W_N \Big(\frac{1}{\ti n}\sum_{j=1}^{\ti n} \ti I_j \ti X_j^{\top}\Big)\big(\ivs-\beta^*\big) \\
    &\le 2\lambda_N \big(\|\beta^*\|_1-\|\ivs\|_1\big) + 2\big\|\frac{1}{\ti n}\sum_{j=1}^{\ti n} \ti X_j \ti I_j^{\top}W_N \hat g_N(\beta^*)\big\|_{\infty}\|\ivs-\beta^*\|_1.
\end{align}
By \Cref{lem:scorebound-finite-m}
\begin{equation}\label{eq:scorebound}
    \big\|\frac{1}{\ti n}\sum_{j=1}^{\ti n} \ti X_j \ti I_j^{\top}W_N \hat g_N(\beta^*)\big\|_{\infty} = O_{p}\Big(\sqrt{\frac{1}{N}}\Big).
\end{equation}
Choose the constant in $\lambda_N$ so that the right side of \eqref{eq:scorebound} is bounded by $$\frac{\lambda_N}{2}$$ with high probability. Then the usual cone condition 
$$
    \|\Delta_{S^c}\|_1 \le 3\|\Delta_S\|_1
$$
holds for 
$$
\Delta\coloneqq\ivs-\beta^*.
$$
By \Cref{ass:RE}~\ref{ass:RE:0}, \Cref{ass:RE}~\ref{ass:RE:3} and Slutsky's theorem we have
$$
    \widehat{\cov}(\ti I,\ti X)^{\top}W_N \widehat{\cov}(\ti I,\ti X) \overset{p}{\to} \cov(\ti I,\ti X)^{\top}W_0 \cov(\ti I,\ti X)
$$
(point-wise and therefore also in operator norm).
Therefore, by \Cref{ass:RE}\ref{ass:RE:1}, we have, for large enough $N$ on the cone $\|\Delta_{S^c}\|_1 \le 3\|\Delta_S\|_1$, that
\begin{equation}
\begin{aligned}
    \kappa \|\Delta\|_2^2 / 2 
    &\leq \Delta^{\top}\Big(\frac{1}{\ti n}\sum_{j=1}^{\ti n} \ti X_j \ti I_j^{\top}\Big)W_N \Big(\frac{1}{\ti n}\sum_{j=1}^{\ti n} \ti I_j \ti X_j^{\top}\Big)\Delta \\
    &\leq 3\lambda_N \|\Delta\|_1 \\
    &\leq 3\lambda_N \sqrt{s^*}\|\Delta\|_2,
\end{aligned}
\end{equation}
which gives
$$
    \|\Delta\|_2 \in O_{p}\Big(\sqrt{\frac{s^*}{N}}\Big)
$$
and then
$$
    \|\Delta\|_1 \in O_{p}\Big(s^*\sqrt{\frac{1}{N}}\Big).
$$

\subsection[Proof of Theorem~\ref{thm:oracle-CI}]{Proof of \Cref{thm:oracle-CI}}
\label{sec:oracle-CI}

By \Cref{thm:ratesSparse} we have $\mathbb{P}(\hat S(W_N) = S^*) \to 1$. On the event $\{\widehat S=S^*\}$, the refit equals the oracle GMM estimator on $S^*$. Standard GMM theory (see \Cref{sec:dense} and note that \Cref{ass:RE}~\ref{ass:RE:1} implies \Cref{ass:asympDense}~\ref{ass:2:1} on $S^*$, i.e.\ ,we have $\rank{\cov(I, X)_{S^*}} = s^*$) yields
$\sqrt{N}(\ivsf_{S^*}(W'_N)-\beta^*_{S^*})\overset{d}{\to} \mathcal N(0,V_{S^*})$ and consistency of the sandwich
$\widehat V_{\widehat S}$. Slutsky’s lemma and $\mathbb{P}(\widehat S=S^*)\to1$ transfer the oracle limit to the random-support estimator.

\subsection{Proof for the Closed Form Solution for Averaging Splits}
\label{sec:proof_splits}
    
    Write $\widehat{\cov}_A(\ti I,\ti X)=\frac{2}{n}\sum_{i\in A} g_i$ and $\widehat{\cov}_B(\ti I,\ti X)=\frac{2}{n}\sum_{j\in B} g_j$, hence
    \begin{equation*}
    \widehat{\cov}_A(\ti I,\ti X)^{\top}\widehat{\cov}_B(\ti I,\ti X)
    =
    \frac{4}{n^2}\sum_{i\in A}\sum_{j\in B} g_i^{\top}g_j.
    \end{equation*}
    For a uniform random split, each ordered pair $(i,j)$ with $i\neq j$ lands in $(A,B)$ with probability
    \begin{equation*}
    \mathbb{P}(i\in A,  j\in B)=\frac{n/2}{n}\cdot\frac{n/2}{n-1}=\frac{n}{4(n-1)}.
    \end{equation*}
    Taking conditional expectation over the split therefore gives
    \begin{equation*}
    \mathbb{E}\big[\widehat{\cov}_A(\ti I,\ti X)^{\top}\widehat{\cov}_B(\ti I,\ti X)\mid (g_i)_{i=1}^n\big]
    =
    \frac{4}{n^2}\cdot\frac{n}{4(n-1)}\sum_{i\neq j} g_i^{\top}g_j
    =
    \frac{1}{n(n-1)}\sum_{i\neq j} g_i^{\top}g_j,
    \end{equation*}
    and rewriting $\sum_{i\neq j} g_i^{\top}g_j=(\sum_i g_i)^{\top}(\sum_j g_j)-\sum_i g_i^{\top}g_i$ yields the displayed closed form.

\subsection[ds]{Proof of \Cref{lem:inconst}}
\label{sec:proof_inconst}
Write
\begin{equation*}
    c\coloneqq \cov(I,X)\in\mathbb R^{m\times d},
    \qquad
    s\coloneqq \cov(I,Y)\in\mathbb R^{m},
    \qquad
    \bar c\coloneqq \sqrt m c,
    \qquad
    \bar s\coloneqq \sqrt m s .
\end{equation*}
Define
\begin{equation*}
    \hat c_k \coloneqq \sqrt m \widehat{\cov}_k(\ti I,\ti X),
    \qquad
    E_k\coloneqq \hat c_k-\bar c
    = \sqrt m\Big(\widehat{\cov}_k(\ti I,\ti X)-c\Big),
    \qquad k\in[K].
\end{equation*}
Define
\begin{equation*}
    \hat c \coloneqq \sqrt m \widehat{\cov}(\ti I,\ti X),
    \qquad
    E\coloneqq \hat c-\bar c
    = \sqrt m\Big(\widehat{\cov}(\ti I,\ti X)-c\Big).
\end{equation*}
For the $(I,Y)$ sample define
\begin{equation*}
    \hat s \coloneqq \sqrt m \widehat{\cov}(I,Y),
    \qquad
    H\coloneqq \hat s-\bar s
= \sqrt m\Big(\widehat{\cov}(I,Y)-s\Big).
\end{equation*}

We have
\begin{align*} 
    \hat \beta &\coloneqq \frac{\widehat{\cov}(\ti I, \ti X)^{\top} \widehat{\cov}(I, Y)}{\widehat{\cov}(\ti I, \ti X)^{\top} \widehat{\cov}(\ti I, \ti X)}  \nonumber \\
    &= \frac{
    \bar c^\top \bar s
    +
    \bar c^\top H
    +
    E^\top \bar s
    +
    E^\top H}
    {\bar c^\top \bar c
    +
    \bar c^\top E
    +
    E^\top \bar c
    +
    E^\top E}.
\end{align*}

By \Cref{lem:scalar-rates-scaled} we have that $\bar c^\top H
    +
    E^\top \bar s
    +
    E^\top H = o_p(1)$ and that
$\bar c^\top E
    +
    E^\top \bar c = o_p(1)$
and $E^\top E = b / \ti r + o_p(1)$.
Furthermore, it holds that $\bar s = \bar c \beta^*$. Therefore, $\hat{\beta} \to \beta^* \frac{Q}{Q + b / \ti r}$ in probability.

\subsection[5]{Proof of \Cref{thm:cf-gmm-weak}}
\label{sec:cf-gmm-weak}
Write
\begin{equation*}
    c\coloneqq \cov(I,X)\in\mathbb R^{m\times d},
    \qquad
    s\coloneqq \cov(I,Y)\in\mathbb R^{m},
    \qquad
    \bar c\coloneqq \sqrt m c,
    \qquad
    \bar s\coloneqq \sqrt m s .
\end{equation*}
Define
\begin{equation*}
    \hat c_k \coloneqq \sqrt m \widehat{\cov}_k(\ti I,\ti X),
    \qquad
    E_k\coloneqq \hat c_k-\bar c
    = \sqrt m\Big(\widehat{\cov}_k(\ti I,\ti X)-c\Big),
    \qquad k\in[K].
\end{equation*}
Define
\begin{equation*}
    \hat c \coloneqq \sqrt m \widehat{\cov}(\ti I,\ti X),
    \qquad
    E\coloneqq \hat c-\bar c
    = \sqrt m\Big(\widehat{\cov}(\ti I,\ti X)-c\Big).
\end{equation*}
For the $(I,Y)$ sample define
\begin{equation*}
    \hat s \coloneqq \sqrt m \widehat{\cov}(I,Y),
    \qquad
    H\coloneqq \hat s-\bar s
= \sqrt m\Big(\widehat{\cov}(I,Y)-s\Big).
\end{equation*}

\begin{lemma}
\label{lem:scalar-rates-scaled}
    Assume \Cref{ass:hd-weak}. Let $d\in\mathbb{N}$ and $K\ge 2$ be fixed, and $n/m\to r \in (0, \infty)$, $\ti n/m\to\ti r \in (0, \infty)$.
    Then as $m\to\infty$:
    \begin{enumerate}[label=(\roman*), nosep]
        \item For each fixed $k\in[K]$,
        \begin{equation*}
            \|\bar c^\top E_k\|_{\mathrm{op}} = O_p(m^{-1/2}),
            \qquad
            \|\bar c^\top E\|_{\mathrm{op}} = O_p(m^{-1/2}).
        \end{equation*}
        \item For $h\ne k$,
        \begin{equation*}
            \|E_h^\top E_k\|_{\mathrm{op}}=O_p(m^{-1/2}).
        \end{equation*}
        \item Cross-sample terms vanish:
        \begin{equation*}
            \|E_k^\top H\|_2=O_p(m^{-1/2}),
            \qquad
            \|E^\top H\|_2=O_p(m^{-1/2}),
            \qquad
            \|\bar c^\top H\|_2=O_p(m^{-1/2}).
        \end{equation*}
        \item If additionally $d=1$, $\tr(\Sigma_{IX})\to b\in(0,\infty)$ and $\sup_m \mathbb{E} [\|IX - c\|_2^4] \leq \infty$, then
        \begin{equation*}
            \|E\|_2^2 \overset{p}{\to} \frac{b}{\ti r}.
        \end{equation*}
    \end{enumerate}
\end{lemma}

\begin{proof}
    Throughout, $d$ and $K$ are fixed and $\ti n_k=\ti n/K\asymp m$.
    By construction, $\widehat{\cov}_k(\ti I,\ti X)$ is the average of $\ti n_k$ i.i.d.\ copies of $\ti I \ti X$, hence
    \begin{equation*}
        \mathbb E[E_k]=0,
        \qquad
        \var(\mathrm{vec}(E_k))=\frac{m}{\ti n_k}\Sigma_{IX}.
    \end{equation*}
    Similarly,
    \begin{equation*}
        \mathbb E[E]=0,
        \qquad
        \var(\mathrm{vec}(E))=\frac{m}{\ti n}\Sigma_{IX}.
    \end{equation*}
    Moreover, since $\widehat{\cov}(I,Y)$ averages $n$ i.i.d.\ copies of $IY$,
    \begin{equation*}
        \mathbb E[H]=0,
        \qquad
        \var(H)=\frac{m}{n}\Sigma_{IY}.
    \end{equation*}
    Finally, under \Cref{ass:hd-weak}~\ref{ass:hd-weak:23} we have $s=c\beta^*$ and thus $\bar s=\bar c \beta^*$.
    
    \paragraph{(i).}
    Fix unit vectors $u,v\in\mathbb R^d$. Then
    \begin{equation*}
        v^\top \bar c^\top E_k u
        =
        (u\otimes \bar c v)^\top \mathrm{vec}(E_k),
    \end{equation*}
    and therefore
    \begin{align*}
        \var\big(v^\top \bar c^\top E_k u\big)
        &=
        (u\otimes \bar c v)^\top \var(\mathrm{vec}(E_k))(u\otimes \bar c v)\\
        &\le
        \|\var(\mathrm{vec}(E_k))\|_{\mathrm{op}}\|u\otimes \bar c v\|_2^2\\
        &=
        \frac{m}{\ti n_k}\|\Sigma_{IX}\|_{\mathrm{op}}\|\bar c v\|_2^2\\
        &\le
        \frac{m}{\ti n_k}\frac{C}{m}\|\bar c\|_{\mathrm{op}}^2
        =
        O(m^{-1}),
    \end{align*}
    where we used $m\|\Sigma_{IX}\|_{\mathrm{op}}\le C$ and $\ti n_k\asymp m$. Hence
    $v^\top \bar c^\top E_k u=O_p(m^{-1/2})$ for each fixed $u,v$.
    Since $d$ is fixed, $\|\bar c^\top E_k\|_{\mathrm{op}}=O_p(m^{-1/2})$.
    The same argument yields $\|\bar c^\top E\|_{\mathrm{op}}=O_p(m^{-1/2})$.
    
    \paragraph{(ii).}
    Fix unit vectors $u,v\in\mathbb R^d$ and $h\ne k$. Because folds are independent,
    $\mathbb E[u^\top(E_h^\top E_k)v]=0$ and
    \begin{equation*}
        u^\top(E_h^\top E_k)v
        =
        (E_h u)^\top(E_k v).
    \end{equation*}
    For independent mean-zero vectors $A,B\in\mathbb R^m$,
    $\var(A^\top B)=\tr(\var(A)\var(B))$, so
    \begin{equation*}
        \var\big(u^\top(E_h^\top E_k)v\big)
        =
        \tr\big(\var(E_h u)\var(E_k v)\big)
        \le
        m\|\var(E_h u)\|_{\mathrm{op}}\|\var(E_k v)\|_{\mathrm{op}}.
    \end{equation*}
    Now $\var(\mathrm{vec}(E_h))=\frac{m}{\ti n_h}\Sigma_{IX}$ implies
    $\|\var(E_h u)\|_{\mathrm{op}}\le \frac{m}{\ti n_h}\|\Sigma_{IX}\|_{\mathrm{op}}$ and similarly for $E_k v$, hence
    \begin{equation*}
        \var\big(u^\top(E_h^\top E_k)v\big)
        \le
        m\Big(\frac{m}{\ti n_h}\|\Sigma_{IX}\|_{\mathrm{op}}\Big)\Big(\frac{m}{\ti n_k}\|\Sigma_{IX}\|_{\mathrm{op}}\Big)
        \le
        m\Big(\frac{m}{\Theta(m)}\frac{C}{m}\Big)^2
        =
        O(m^{-1}).
    \end{equation*}
    Therefore $u^\top(E_h^\top E_k)v=O_p(m^{-1/2})$ for each fixed $u,v$,
    and since $d$ is fixed this yields $\|E_h^\top E_k\|_{\mathrm{op}}=O_p(m^{-1/2})$.
    
    \paragraph{(iii).}
    We show $\|E^\top H\|_2=O_p(m^{-1/2})$; the other statements are analogous.
    Fix a unit vector $u\in\mathbb R^d$. By independence of the $(\ti I,\ti X)$ and $(I,Y)$ samples,
    $\mathbb E[u^\top E^\top H]=0$ and conditioning on $E$ gives
    \begin{align*}
        \var(u^\top E^\top H\mid E)
        &=
        u^\top E^\top \var(H) E u
        =
        u^\top E^\top \Big(\frac{m}{n}\Sigma_{IY}\Big)E u\\
        &\le
        \frac{m}{n}\|\Sigma_{IY}\|_{\mathrm{op}}\|Eu\|_2^2
        \le
        \frac{m}{n}\frac{C}{m}\|E\|_{\mathrm{op}}^2.
    \end{align*}
    Taking expectation and using $\mathbb E\|E\|_{\mathrm{op}}^2\le \mathbb E\|E\|_F^2
    =\tr(\var(\mathrm{vec}(E)))=\frac{m}{\ti n}\tr(\Sigma_{IX})
    \le \frac{m}{\ti n}\cdot md \|\Sigma_{IX}\|_{\mathrm{op}}
    \le \frac{m}{\ti n}\cdot md\cdot \frac{C}{m}
    =O(1)$ yields
    \begin{equation*}
        \var(u^\top E^\top H)=O(m^{-1}).
    \end{equation*}
    Thus $u^\top E^\top H=O_p(m^{-1/2})$ for each fixed $u$, and since $d$ is fixed,
    $\|E^\top H\|_2=O_p(m^{-1/2})$.
    The bound for $\|E_k^\top H\|_2$ follows by the same argument with $\ti n$ replaced by $\ti n_k$.
    Finally, $\|\bar c^\top H\|_2=O_p(m^{-1/2})$ follows from
    \begin{equation*}
        \var(\bar c^\top H)=\bar c^\top \var(H)\bar c
        =
        \bar c^\top\Big(\frac{m}{n}\Sigma_{IY}\Big)\bar c
        \le
        \frac{m}{n}\|\Sigma_{IY}\|_{\mathrm{op}}\|\bar c\|_F^2
        \le
        \frac{m}{\Theta(m)}\frac{C}{m}\|\bar c\|_F^2
        =
        O(m^{-1}),
    \end{equation*}
    using $\|\bar c\|_2^2=m\|c\|_2^2=\tr(m c^\top c)=O(1)$.
    
    \paragraph{(iv).}
    Assume $d=1$ and write $Z_j\coloneqq IX-c\in\mathbb R^m$ for the centered summands in the first-stage covariance.
    Then
    \begin{equation*}
        E
        =
        \sqrt m\Big(\frac{1}{\ti n}\sum_{j=1}^{\ti n} Z_j\Big),
        \qquad
        \|E\|_2^2
        =
        \frac{m}{\ti n^2}\sum_{j=1}^{\ti n}\|Z_j\|_2^2
        +
        \frac{m}{\ti n^2}\sum_{i\ne j} Z_i^\top Z_j
        \eqqcolon A_m + B_m.
    \end{equation*}
    We have
    \begin{equation*}
        \mathbb E[A_m]
        =
        \frac{m}{\ti n}\mathbb E\|Z_1\|_2^2
        =
        \frac{m}{\ti n}\tr(\Sigma_{IX})
        \to
        \frac{b}{\ti r}.
    \end{equation*}
    Moreover, under uniformly bounded fourth moments ($\sup_m \mathbb{E} [\|IX - c\|_2^4] \leq \infty$) of $\|Z_1\|_2$,
    $\var(A_m)=O(1/m)$, hence $A_m-\mathbb E[A_m]=o_p(1)$.
    
    Next, $\mathbb E[B_m]=0$ by independence and centering.
    For the variance of $B_m$, write
    \begin{equation*}
        S \coloneqq \sum_{i\ne j} Z_i^\top Z_j = 2\sum_{1\le i<j\le \ti n} Z_i^\top Z_j.
    \end{equation*}
    We claim that the summands $Z_i^\top Z_j$ are pairwise uncorrelated across distinct unordered pairs.
    Indeed, if $(i,j)\ne (k,\ell)$ as unordered pairs, then either the pairs are disjoint, in which case independence gives zero covariance, or they share exactly one index, say $i=k$ and $j\ne \ell$, in which case
    \begin{equation*}
        \mathbb E[(Z_i^\top Z_j)(Z_i^\top Z_\ell)]
        = \mathbb E\!\left[\,Z_i^\top \,\mathbb E[Z_j]\, Z_i^\top \,\mathbb E[Z_\ell]\,\right]=0
    \end{equation*}
    since $Z_j,Z_\ell$ are independent of $Z_i$ and mean zero. Hence
    \begin{equation*}
        \operatorname{Var}\Big(\sum_{i<j} Z_i^\top Z_j\Big)
        = \binom{\ti n}{2}\operatorname{Var}(Z_1^\top Z_2).
    \end{equation*}
    Moreover, by independence and centering,
    \begin{equation*}
        \operatorname{Var}(Z_1^\top Z_2)=\mathbb E[(Z_1^\top Z_2)^2]
        = \operatorname{tr}(\Sigma_{IX}^2).
    \end{equation*}
    Therefore
    \begin{equation*}
        \operatorname{Var}(S)=4\binom{\ti n}{2}\operatorname{tr}(\Sigma_{IX}^2)
        =2\ti n(\ti n-1)\operatorname{tr}(\Sigma_{IX}^2),
    \end{equation*}
    and thus
    \begin{equation*}
        \operatorname{Var}(B_m)
        =\Big(\frac{m}{\ti n^2}\Big)^2 \operatorname{Var}(S)
        =2\frac{m^2(\ti n-1)}{\ti n^3}\operatorname{tr}(\Sigma_{IX}^2)
        =O\Big(\frac{m^2}{\ti n^2}\operatorname{tr}(\Sigma_{IX}^2)\Big).
    \end{equation*}
    Since $\ti n\asymp m$, this is $O(\operatorname{tr}(\Sigma_{IX}^2))$.

    Since $\tr(\Sigma_{IX}^2)\le md \|\Sigma_{IX}\|_{\mathrm{op}}^2\le md (C/m)^2=O(1/m)\to 0$, we get $B_m=o_p(1)$.
    Combining yields $\|E\|_2^2=A_m+B_m\overset{p}{\to} b/\ti r$.
\end{proof}

Recall the normalized cross-moment definitions
\begin{equation*}
    C_{XX}
    \coloneqq
    \frac{m}{K(K-1)}\sum_{h\ne k}\widehat{\cov}_h(\ti I,\ti X)^\top \widehat{\cov}_k(\ti I,\ti X),
    \qquad
    C_{XY}
    \coloneqq
    m \widehat{\cov}(\ti I,\ti X)^\top \widehat{\cov}(I,Y).
\end{equation*}
Equivalently, with $\hat c_k=\sqrt m \widehat{\cov}_k(\ti I,\ti X)$, $\hat c=\sqrt m \widehat{\cov}(\ti I,\ti X)$ and $\hat s=\sqrt m \widehat{\cov}(I,Y)$,
\begin{equation*}
    C_{XX}
    =
    \frac{1}{K(K-1)}\sum_{h\ne k}\hat c_h^\top \hat c_k,
    \qquad
    C_{XY}
    =
    \hat c^\top \hat s.
\end{equation*}
Let $\bar c=\sqrt m c$ and $\bar s=\sqrt m s$. Under \Cref{ass:1}, $s=c\beta^*$ and thus $\bar s=\bar c \beta^*$.
Write $\hat c_k=\bar c+E_k$, $\hat c=\bar c+E$, and $\hat s=\bar s+H$ as in Lemma~\ref{lem:scalar-rates-scaled}.

Expand
\begin{align*}
    C_{XX}
    &=
    \frac{1}{K(K-1)}\sum_{h\ne k}(\bar c+E_h)^\top(\bar c+E_k)\\
    &=
    \bar c^\top \bar c
    +
    \frac{1}{K(K-1)}\sum_{h\ne k}\bar c^\top E_k
    +
    \frac{1}{K(K-1)}\sum_{h\ne k}E_h^\top \bar c
    +
    \frac{1}{K(K-1)}\sum_{h\ne k}E_h^\top E_k.
\end{align*}
By Lemma~\ref{lem:scalar-rates-scaled}(i) and (ii), each of the last three terms is $o_p(1)$ since $K$ is fixed.
Therefore
\begin{equation*}
    C_{XX}
    =
    \bar c^\top \bar c + o_p(1)
    =
    m c^\top c + o_p(1)
    \overset{p}{\to} Q.
\end{equation*}
Expand
\begin{align*}
    C_{XY}
    &=
    (\bar c+E)^\top(\bar s+H)
    =
    \bar c^\top \bar s
    +
    \bar c^\top H
    +
    E^\top \bar s
    +
    E^\top H.
\end{align*}
Since $\bar s=\bar c \beta^*$,
\begin{equation*}
    \bar c^\top \bar s
    =
    \bar c^\top \bar c \beta^*
    =
    m c^\top c \beta^*
    \to
    Q\beta^*.
\end{equation*}
Moreover, Lemma~\ref{lem:scalar-rates-scaled}(iii) gives $\bar c^\top H=o_p(1)$ and $E^\top H=o_p(1)$, and Lemma~\ref{lem:scalar-rates-scaled}(i) gives $E^\top \bar s=(\bar c^\top E)^\top\beta^*=o_p(1)$.
Hence $C_{XY}\overset{p}{\to}Q\beta^*$.

By assumption $W_N\overset{p}{\to}W_0\succ 0$, and we have shown that $C_{XX}\overset{p}{\to}Q\succ 0$.
Therefore
\begin{equation*}
(C_{XX}^\top W_N^{-1}C_{XX})^{-1}C_{XX}^\top W_N^{-1}
\overset{p}{\to}
(Q^\top W_0^{-1}Q)^{-1}Q^\top W_0^{-1}.
\end{equation*}
    Multiplying by $C_{XY}\overset{p}{\to}Q\beta^*$ yields
\begin{equation*}
    \hat\beta^{\mathrm{UP},\mathrm{HD}}_{\mathrm{GMM}}(W_N)
    =
    (C_{XX}^\top W_N^{-1}C_{XX})^{-1}C_{XX}^\top W_N^{-1}C_{XY}
    \overset{p}{\to}\beta^*.
\end{equation*}

\subsection[k]{Proof of \Cref{thm:rates-large-m-sparse-weak}}
\label{sec:rates-large-m-sparse-weak}
This proof follows the same ideas as the proof of \Cref{thm:ratesSparse}.

\begin{lemma}
\label{lem:scorebound-large-m}
    Under Assumption~\ref{ass:RE-large-m-weak},
    \begin{equation*}
        \|C_{XY}-C_{XX}\beta^*\|_2 = O_p(m^{-1/2}),
    \end{equation*}
    and consequently
    \begin{equation*}
        \big\|C_{XX}^\top W_m\big(C_{XY}-C_{XX}\beta^*\big)\big\|_\infty = O_p(m^{-1/2}).
    \end{equation*}
\end{lemma}

\begin{proof}
    Write $c\coloneqq \cov(I,X)\in\mathbb R^{m\times d}$ and $s\coloneqq \cov(I,Y)\in\mathbb R^{m}$.
    Set $\bar c\coloneqq \sqrt m\,c$ and $\bar s\coloneqq \sqrt m\,s$. Under Assumption~\ref{ass:1},
    $s=c\beta^*$ and thus $\bar s=\bar c\,\beta^*$.
    
    Let $\hat c \coloneqq \sqrt m\,\widehat{\cov}(\ti I,\ti X)$ and $\hat s \coloneqq \sqrt m\,\widehat{\cov}(I,Y)$.
    For fold $k$, let $\hat c_k \coloneqq \sqrt m\,\widehat{\cov}_k(\ti I,\ti X)$.
    Write the centered errors as
    \begin{equation*}
        \hat c=\bar c+E,
        \qquad
        \hat s=\bar s+H,
        \qquad
        \hat c_k=\bar c+E_k.
    \end{equation*}
    Then $C_{XY}=\hat c^\top \hat s$ and $C_{XX}=\frac{1}{K(K-1)}\sum_{h\ne k}\hat c_h^\top \hat c_k$.
    Expanding,
    \begin{equation*}
        C_{XY}
        =(\bar c+E)^\top(\bar s+H)
        =\bar c^\top\bar c\,\beta^* + \bar c^\top H + (E^\top\bar c)\beta^* + E^\top H,
    \end{equation*}
    and
    \begin{equation*}
        C_{XX}
        =\bar c^\top\bar c + R_m,
        \qquad
        R_m \coloneqq \frac{1}{K(K-1)}\sum_{h\ne k}\Big(\bar c^\top E_k + E_h^\top\bar c + E_h^\top E_k\Big).
    \end{equation*}
    Therefore,
    \begin{equation*}
        C_{XY}-C_{XX}\beta^*
        =
        \bar c^\top H + (E^\top\bar c)\beta^* + E^\top H - R_m\beta^*.
    \end{equation*}
    By Lemma~\ref{lem:scalar-rates-scaled}(iii), $\|\bar c^\top H\|_2=O_p(m^{-1/2})$ and $\|E^\top H\|_2=O_p(m^{-1/2})$.
    By Lemma~\ref{lem:scalar-rates-scaled}(i), $\|\bar c^\top E\|_{\mathrm{op}}=O_p(m^{-1/2})$, hence
    \begin{equation*}
        \|(E^\top\bar c)\beta^*\|_2
        \le \|\bar c^\top E\|_{\mathrm{op}}\|\beta^*\|_2
        = O_p(m^{-1/2})
    \end{equation*}
    since $\|\beta^*\|_2=O(1)$.
    Finally, by Lemma~\ref{lem:scalar-rates-scaled}(i) and (ii), each summand in $R_m$ is $O_p(m^{-1/2})$ in operator norm, and since $K$ is fixed, the finite average satisfies $\|R_m\|_{\mathrm{op}}=O_p(m^{-1/2})$, hence $\|R_m\beta^*\|_2=O_p(m^{-1/2})$.
    By the triangle inequality,
    \begin{equation*}
        \|C_{XY}-C_{XX}\beta^*\|_2 = O_p(m^{-1/2}).
    \end{equation*}
    
    Next, since $C_{XX}\overset{p}{\to}Q$ and $W_m\overset{p}{\to}W_0$, we have $\|C_{XX}\|_{\mathrm{op}}=O_p(1)$ and $\|W_m\|_{\mathrm{op}}=O_p(1)$.
    Therefore,
    \begin{equation*}
        \big\|C_{XX}^\top W_m(C_{XY}-C_{XX}\beta^*)\big\|_\infty
        \le \big\|C_{XX}^\top W_m(C_{XY}-C_{XX}\beta^*)\big\|_2
        \le \|C_{XX}\|_{\mathrm{op}}\|W_m\|_{\mathrm{op}}\|C_{XY}-C_{XX}\beta^*\|_2
        = O_p(m^{-1/2}).
    \end{equation*}
\end{proof}

Let $H_m\coloneqq C_{XX}^{\top} W_m C_{XX}$ and $b_m\coloneqq C_{XX}^{\top}W_m C_{XY}$. The objective equals $\frac12\beta^{\top}H_m\beta-b_m^{\top}\beta+\frac12\|W_m^{1/2}C_{XY}\|_2^2+\lambda_m\|\beta\|_1$. The basic inequality gives
\begin{equation*}
    \tfrac12\Delta^{\top}H_m\Delta \le \big\|C_{XX}^{\top}W_m(C_{XY}-C_{XX}\beta^*)\big\|_{\infty}\|\Delta\|_1+\lambda_m(\|\beta^*\|_1-\|\beta^*+\Delta\|_1),
\end{equation*}
for $\Delta\coloneqq \hat\beta^{\mathrm{UP}, \mathrm{HD}}_{\mathrm{GMM}, \ell_1}-\beta^*$. By \Cref{lem:scorebound-large-m} and the choice of $\lambda_m$, with probability approaching one the score term is at most $\lambda_m/2$, yielding the cone constraint $\|\Delta_{S^{*c}}\|_1\le 3\|\Delta_{S^*}\|_1$.
On this cone, $H_m\to Q ^{\top}W_0 Q $ (by \Cref{ass:RE-large-m-weak}~\ref{ass:RE-large-m-weak:4} and fixed $d$). \Cref{ass:RE-large-m-weak}~\ref{ass:RE-large-m-weak:2} then implies $\Delta^{\top}H_m\Delta\gtrsim\kappa\|\Delta\|_2^2$ for large enough $m$. Therefore $\|\Delta\|_2\lesssim\sqrt{s^*}\lambda_m$ and $\|\Delta\|_1\lesssim s^*\lambda_m$, proving the rates. The beta–min condition yields support recovery.

\end{document}